\def\isarxivversion{1} 

\ifdefined\isarxivversion
\documentclass[11pt]{article}
\else
\documentclass{article}
\usepackage{icml2019}
\fi

\usepackage{graphicx}
\usepackage{amsmath}
\usepackage{amsthm}
\usepackage{amssymb}
\usepackage{float}
\usepackage{subfig}
\usepackage{wrapfig, epsfig, epstopdf}
\usepackage{algorithm}
\usepackage{diagbox}
\usepackage[dash,dot]{dashundergaps}
\usepackage{algpseudocode}
\usepackage{multirow}
\usepackage{tikz}
\usepackage{comment}
\usepackage{color}
\usepackage{soul}
\usepackage{hyperref}  
\hypersetup{colorlinks=true,citecolor=red,linkcolor=blue}

\usetikzlibrary{arrows}
\graphicspath{{./fig/}}

\ifdefined\isarxivversion
\usepackage[margin=1in]{geometry}
\else 

\fi

\newtheorem{theorem}{Theorem}[section]
\newtheorem{lemma}[theorem]{Lemma}
\newtheorem{definition}[theorem]{Definition}

\newtheorem{remark}[theorem]{Remark}
\newtheorem{claim}[theorem]{Claim}

\renewcommand{\tilde}{\widetilde}

\newcommand{\wt}{\widetilde}

\newcommand{\ov}{\overline}
\newcommand{\GS}{\mathrm{GS}}
\newcommand{\ap}{\mathrm{ap}}
\newcommand{\dfp}{\mathrm{dp}}
\renewcommand{\d}{\mathrm{d}}
\DeclareMathOperator{\poly}{poly}

\DeclareMathOperator{\R}{{\mathbb R}}

\DeclareMathOperator{\N}{{\mathcal{N}}}
\DeclareMathOperator*{\E}{{\mathbb{E}}}
\DeclareMathOperator*{\var}{\mathrm{Var}}
\DeclareMathOperator*{\Var}{\mathrm{Var}}

\setulcolor{red}

\definecolor{b2}{RGB}{51,153,255}
\definecolor{mygreen}{RGB}{80,180,0}
\definecolor{mycy2}{RGB}{255,51,255}

\begin{document}

\ifdefined\isarxivversion

\title{Privacy-preserving Learning via Deep Net Pruning}

\date{}

\author{
Yangsibo Huang\thanks{\texttt{yangsibo@princeton.edu}. Princeton University.} \hspace{8mm}
\and
Yushan Su\thanks{\texttt{yushans@princeton.edu}. Princeton University.} \hspace{8mm}
\and
Sachin Ravi\thanks{\texttt{sachinravi14@gmail.com}. Princeton University.} \hspace{8mm}
\and
Zhao Song\thanks{\texttt{zhaos@ias.edu}. Princeton University and Institute for Advanced Study.} 
\and
\hspace{5mm}
Sanjeev Arora\thanks{\texttt{arora@cs.princeton.edu}. Princeton University.} \hspace{5mm}
\and
Kai Li\thanks{\texttt{li@cs.princeton.edu}. Princeton University.} \hspace{4mm}
}

\else

\icmltitlerunning{Privacy-preserving Learning via Deep Net Pruning}

\twocolumn[
\icmltitle{Privacy-preserving Learning via Deep Net Pruning}



\icmlsetsymbol{equal}{*}

\begin{icmlauthorlist}
\icmlauthor{Aeiau Zzzz}{equal,to}
\icmlauthor{Bauiu C.~Yyyy}{equal,to,goo}
\icmlauthor{Cieua Vvvvv}{goo}
\icmlauthor{Iaesut Saoeu}{ed}
\icmlauthor{Fiuea Rrrr}{to}
\icmlauthor{Tateu H.~Yasehe}{ed,to,goo}
\icmlauthor{Aaoeu Iasoh}{goo}
\icmlauthor{Buiui Eueu}{ed}
\icmlauthor{Aeuia Zzzz}{ed}
\icmlauthor{Bieea C.~Yyyy}{to,goo}
\icmlauthor{Teoau Xxxx}{ed}
\icmlauthor{Eee Pppp}{ed}
\end{icmlauthorlist}

\icmlaffiliation{to}{Department of Computation, University of Torontoland, Torontoland, Canada}
\icmlaffiliation{goo}{Googol ShallowMind, New London, Michigan, USA}
\icmlaffiliation{ed}{School of Computation, University of Edenborrow, Edenborrow, United Kingdom}

\icmlcorrespondingauthor{Cieua Vvvvv}{c.vvvvv@googol.com}
\icmlcorrespondingauthor{Eee Pppp}{ep@eden.co.uk}

\icmlkeywords{Machine Learning, ICML}

\vskip 0.3in
]

\fi

\ifdefined\isarxivversion
\begin{titlepage}
  \maketitle
  \begin{abstract}
This paper attempts to answer the question whether neural network pruning can be used as a tool to achieve differential privacy without losing much data utility. As a first step towards understanding the relationship between neural network pruning and differential privacy, this paper proves that pruning a given layer of the neural network is equivalent to adding a certain amount of differentially private noise to its hidden-layer activations. The paper also presents experimental results to show the practical implications of the theoretical finding and the key parameter values in a simple practical setting. These results show that neural network pruning can be an effective alternative to adding differentially private noise for neural networks.

  \end{abstract}
  \thispagestyle{empty}
\end{titlepage}

\else

\begin{abstract}
This paper attempts to answer the question whether neural network pruning can be used as a tool to achieve differential privacy without losing much data utility. As a first step towards understanding the relationship between neural network pruning and differential privacy, this paper proves that pruning a given layer of the neural network is equivalent to adding a certain amount of differentially private noise to its hidden-layer activations. The paper also presents experimental results to show the practical implications of the theoretical finding and the key parameter values in a simple practical setting. These results show that neural network pruning can be an effective alternative to adding differentially private noise for neural networks.

\end{abstract}

\fi


\section{Introduction}

Data privacy has become one of the top concerns in machine learning with deep neural networks, since there is an increasing demand to train deep net models on distributed, private data sets.  For example, hospitals are now training their automated diagnosis systems on private patients' data \cite{lst+16, ls17, dlr+18}; and advertisement providers are collecting users' online trajectories to optimize their learning-based recommendation algorithm \cite{cas16,yhc+18}. These private data, however, are usually decentralized in nature, and policies such as  the Health Insurance Portability and Accountability Act (HIPAA) \cite{act1996health} and the California Consumer Privacy Act (CCPA) \cite{ccpa} restrict the exchange of raw data among distributed users.

Various schemes have been proposed for privacy sensitive deep learning with distributed private data, where model updates \cite{kmy+16} or hidden-layer representations \cite{vgsr18} are exchanged  instead of the raw data. However, recent research identified that even if the raw data are kept private, sharing the model updates or hidden-layer activations can still leak sensitive information about the input, which we refer to as the victim. Such leakage can be: the victim's class, the victim's feature \cite{fjr15}, or even its original record \cite{mv15,db16,zlh19}.  Privacy leakage poses a severe threat to individuals whose private records have been collected to train the deep neural network.

Differential privacy (DP) \cite{dmns06} has emerged, during the past few years, as a standard framework to analyze privacy leakage. The core idea of achieving differential privacy is to add controlled noise to the output of a deterministic function. However, there is a long standing trade-off in adding noise to preserve privacy: as privacy leakage decreases, accuracy decreases too. For data without strong signals, an alternative that satisfies privacy guarantees without decreasing accuracy is strongly desired. 

\begin{figure}[h]
    \centering
    \subfloat[0]{\includegraphics[width=0.08\textwidth]{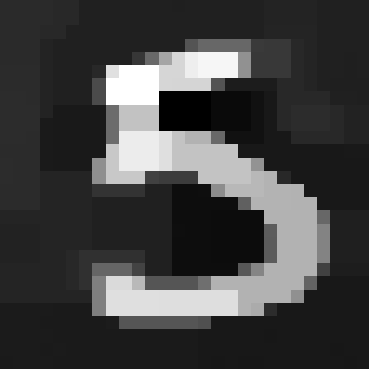}}\hspace{2mm}
    \subfloat[0.1]{\includegraphics[width=0.08\textwidth]{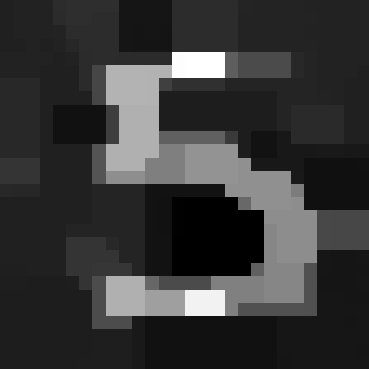}}\hspace{2mm}
    \subfloat[0.2]{\includegraphics[width=0.08\textwidth]{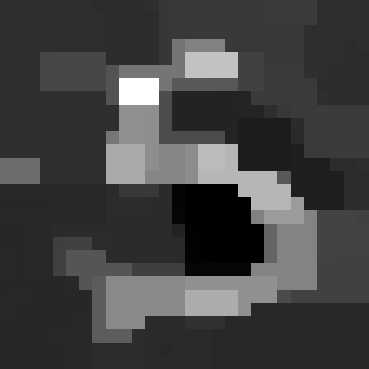}}\hspace{2mm}
    \subfloat[0.3]{\includegraphics[width=0.08\textwidth]{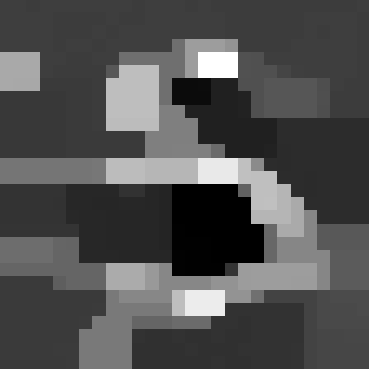}}\hspace{2mm}
    \subfloat[0.4]{\includegraphics[width=0.08\textwidth]{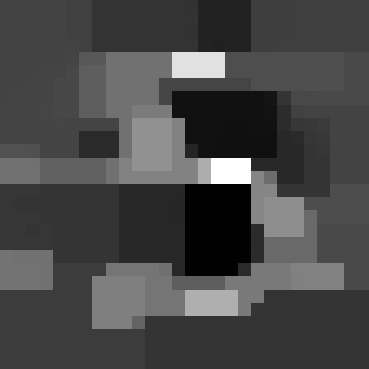}}\hspace{2mm}
    \subfloat[0.5]{\includegraphics[width=0.08\textwidth]{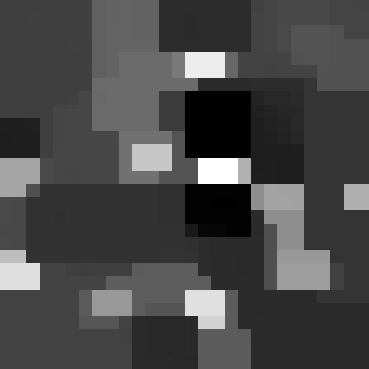}}\hspace{2mm}
    \subfloat[0.6]{\includegraphics[width=0.08\textwidth]{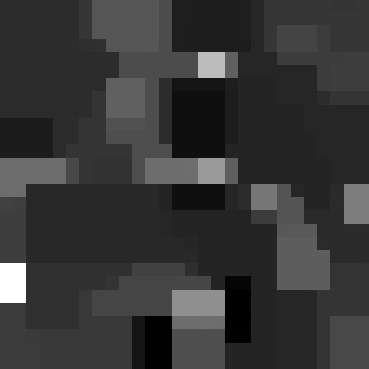}}\hspace{2mm}
    \subfloat[0.7]{\includegraphics[width=0.08\textwidth]{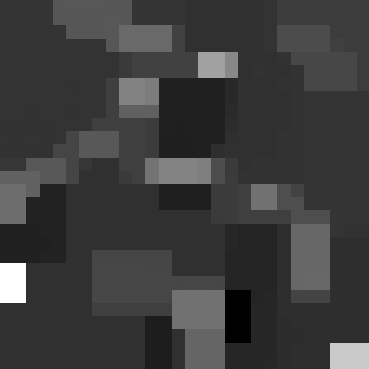}}\hspace{2mm}
    \subfloat[0.8]{\includegraphics[width=0.08\textwidth]{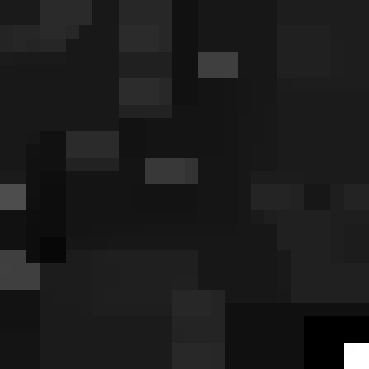}}\hspace{2mm}
    \subfloat[0.9]{\includegraphics[width=0.08\textwidth]{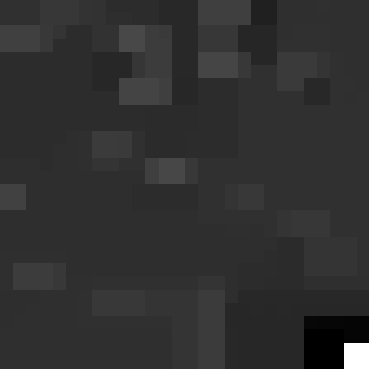}}\hspace{2mm}
    \caption{Visualization of inverted MNIST inputs from the Conv5 layer of LeNet-5 using the inversion algorithm in Section~\ref{sec:inversion_alg}.  We prune the network with different target sparsity, shown in caption (sparsity = the fraction of weights we set to 0). Pruning with higher sparsity seems to leak less information about the input.}
    \label{fig:vis_intro}

\end{figure}

In this work, we attempt to answer the question whether neural network pruning can be such an alternative. Two observations motivate this question.  First, pruning perturbs a neural network, which is somewhat similar to adding noise.  Figure~\ref{fig:vis_intro} shows that neural network pruning makes the inversion from hidden-layers harder as the sparsity of the pruned network increases. Second, \cite{hmd15,lkd+16} showed that they were able to prune a large percentage (e.g. 90$\%$) of the weights in several commonly used neural networks without reducing accuracy.

The approach taken in this paper is to prove the existence of an equivalence between pruning and adding differentially private noise.  By doing so, one can analyze privacy leakage of network pruning using the framework of differential privacy.  We have focused on a distributed learning setting in which a hidden-layer of activations will be used to communicate among distributed sites.  Thus, our study 
 explores the equivalence between neural network pruning and adding differentially private noise to activations from both theoretical and experimental perspectives.

The paper makes four contributions.  First, to the best of our knowledge, this is the first work that draws a connection between neural network pruning and differential privacy from a theoretical perspective. 

Second, we show that magnitude-based pruning algorithm is differentially private, if the width of the neural network is sufficiently wide. (see Section \ref{sec:main}).

Third, we show that in a synthetic setting with a practical differential budget, the width of the neural network needs to be only a few hundreds in order to make the theoretical equivalence hold (see Section \ref{sec:exp}).

Finally, our experiments with MNIST and CIFAR-10 show that, with the same target accuracy, magnitude pruning preserves substantially more privacy than adding random noise (the classical way to provide differential privacy) to neural network. 

\paragraph{Roadmap.}
 The rest of this paper is organized as follow. Section~\ref{sec:backgrounds} presents some backgrounds about differential privacy and neural network pruning. Section~\ref{sec:main} states our main theoretical result which shows the equivalence between differential privacy and magnitude-based pruning. Section~\ref{sec:related} covers relevant work in different privacy and neural network pruning, and also provides the intuition for possible connections between them from a perspective of sparse recovery. In Section \ref{sec:exp}, we run experiments on MNIST and CIFAR-10 datasets and observe that the experimental results match our theoretical findings. Finally, we conclude this work in Section \ref{sec:conclusion}.
 
 Appendix~\ref{sec:app_prob} presents several basic probability tools. Appendix~\ref{sec:app_conc} states some applications of concentration inequalities. Appendix~\ref{sec:app_anti} states some anti-contraction result and its generalization. Appendix~\ref{sec:app_sens} discusses about sensitivity. Appendix~\ref{sec:app_dp} finally proves our main result. Appendix~\ref{sec:app_exp} show several more experimental results.  
\section{Backgrounds}\label{sec:backgrounds}

\paragraph{Notations.}
For a positive integer $n$, we use $[n]$ to denote set $\{1,2,\cdots, n\}$. For vector $x \in \R^n$, we use $\| x \|_1$ to denote $\sum_{i=1}^n |x_i|$, $\| x \|_2$ to denote $( \sum_{i=1}^n x_i^2 )^{1/2}$, $\| x \|_{\infty}$ to denote $\max_{i \in [n]} |x_i|$. We use ${\cal N}(\mu,\sigma^2)$ to denote random Gaussian distribution. For a matrix $A$, we use $\| A \|$ to denote its spectral norm.

This section presents some backgrounds before theoretically establishing the equivalence between magnitude-based pruning and adding differentially private noise in Section~\ref{sec:main}. Section \ref{sec:dp} revisits the notion of $(\epsilon_{\dfp},\delta_{\dfp})$-differential privacy. Section \ref{sec:mag_prune} describes the magnitude pruning algorithm. 

\subsection{Differential privacy}
\label{sec:dp}
The classical definition of differential privacy is shown as follow:
\begin{definition}[$(\epsilon_{\dfp},\delta_{\dfp})$-differential privacy \cite{dkmmn06}]
\label{def:epsilon_delta_dp}
For a randomized function $h(x)$, we say $h(x)$ is $(\epsilon_{\dfp},\delta_{\dfp})$-differential privacy if for all $S \subseteq \mathrm{Range}(h)$ and for all $x,y$ with $\| x - y \|_1 \leq 1$
\begin{align*}
\Pr_h [ h(x) \in S ] \leq \exp(\epsilon_{\dfp}) \cdot \Pr_h [ h(y) \in S ] + \delta_{\dfp}.
\end{align*}
\end{definition}

Definition \ref{def:epsilon_delta_dp} says that, if there are two otherwise identical records $x$ and $y$, one with privacy-sensitive information in it, and one without it, and we normalize them such that $\|x-y\|_1 \leq 1$. Differential Privacy ensures that the probability that a statistical query will produce a given result is nearly the same whether it’s conducted on the first or second record. Parameters $(\epsilon_{\dfp},\delta_{\dfp})$ are called the privacy budget, and smaller $\epsilon_{\dfp}$ and $\delta_{\dfp}$ provide a better differential privacy protection. One can think of a setting where both parameters are 0, then the chance of telling whether a query result is from $x$ or from $y$ is no better than a random guessing.

A standard strategy to achieve differential privacy is by adding noise to the the original data $x$ or the function output $h(x)$. In order to analyze it, we need the following definition: 

\begin{definition}[Global Sensitivity \cite{dmns06}] 
\label{def:lp_sen}
Let $f : \mathbb{R}^n \rightarrow \mathbb{R}^d$, define $\GS_p(f)$, the $\ell_p$ global sensitivity of $f$, for all $x,y$ with $\| x - y \|_1 \leq 1$ as
\begin{align*}
     \GS_p(f) = \sup_{ x, y \in \mathbb{R}^n } \| f(x) - f(y) \|_p .
\end{align*}
\end{definition}
The global sensitivity of a function measures how `sensitive' the function is to slight changes in input. The noise needed for differential privacy guarantee is then calibrated using some well-known mechanisms, e.g., Laplace or Gaussian \cite{dr14}, and the amount of noise (the standard deviation of the noise distribution) is proportional to the sensitivity, but inversely proportional to the privacy budget $\epsilon_{\dfp}$. That is to say, for a given function with fixed global sensitivity, a larger amount of noise is required to guarantee a better differential privacy (one with a smaller budget $\epsilon_{\dfp}$). 

\subsection{Magnitude-based pruning}
\label{sec:mag_prune}
\begin{algorithm}[t]
    \caption{Stochastic Gradient Descent with Magnitude-based Pruning, simplified version of Algorithm~\ref{alg:mag_prune_full} 
    }
    \label{alg:mag_prune}
    \begin{algorithmic}[1]{
    \Procedure{\textsc{SGDMagPrune}}{$\cal D$, $a$}
    \State $W^{(1)}$ is a random init. of neural network's weights
    \For{$t = 1 \to T_{\text{train}} + T_{\text{prune}}$ } \Comment{Training stage} \label{lin:train_start}
        \State Sample $(x,y)\sim {\cal D}$ uniformly at random
        \If{$t \in T_{\text{prune}}$} \Comment{Pruning stage} \label{lin:prune_start}
          \State  $\wt{W}^{(t)} \leftarrow \textsc{ThPrune}(W^{(t)},a^{(t)})$
        \EndIf
        \State Update $W^{(t+1)}$ based on $\wt{W}^{(t)}$ and gradient
    \EndFor \label{lin:train_end} \label{lin:prune_end}
    \State $T_{\text{end}} \gets T_{\text{train}}+T_{\text{prune}}$
    \State $\wt{W}^{(T_{\text{end}})} \leftarrow \textsc{ThPrune}(W^{(T_{\text{end}})},a^{ ( T_{\text{end}} ) })$
  
    \EndProcedure
    \Procedure{\textsc{ThPrune}}{$W,a$} 
        \State $(\tilde{W}_l)_{i,j}  \leftarrow (W_l)_{i,j} \cdot {\bf 1}_{  |(W_l)_{i,j}| > a }$, $\forall i, j$
        \State \Return $\wt{W}$
    \EndProcedure}
    \end{algorithmic}
\end{algorithm}

Procedure \textsc{SGDMagPrune} (see Algorithm \ref{alg:mag_prune}) describes the process of training a deep neural network with stochastic gradient descent and magnitude-based pruning. 

As shown, the procedure begins with a standard training stage of $T_{\mathrm{train}}$ iterations (line~\ref{lin:train_start} to line~\ref{lin:train_end}). After that, the model enters the pruning stage of $T_{\mathrm{prune}}$ iterations (line~\ref{lin:prune_start} to line~\ref{lin:prune_end}). Inside each iteration of the pruning stage, we firstly perform a layer-wise threshold pruning (see procedure \textsc{ThPrune} in Algorithm~\ref{alg:mag_prune}) which sets the weights with magnitudes smaller than $a$ to zero. Then, we run model update once. At the end of the prune stage, we perform a layer-wise pruning again to guarantee that the resulted weight matrix achieves a certain sparsity (fraction of zeros).

Note that inside each pruning iteration above, we perform magnitude pruning with the threshold $a^{(t)}$. In practice, $a^{(t)}$ is determined by three factors: the target sparsity of the matrix after pruning, the total number of pruning iterations $T_{\text{prune}}$, and the current number of pruning iterations $t$. Since the proof only cares about the final state of the weight matrix, we leave the details of how to dynamically configure $a^{(t)}$  to the experiment section.

\section{Main result}\label{sec:main}

We start by formulating the equivalence between pruning and adding differentially private noise. We propose the following notion to describe the closeness between a randomized function and a given function (either randomized or deterministic).
\begin{definition}[$(\epsilon_{\ap},\delta_{\ap})$-close]\label{def:epsilon_delta_close}
For a pair of functions $g : \R^d \rightarrow \R^m$ and $h : \R^d \rightarrow \R^m$, and a fixed input $x$, we say $g(x)$ is $(\epsilon,\delta)$-close to $h(x)$ if and only if, 
\begin{align*}
\Pr_{g,h} \left[ \frac{1}{\sqrt{m}} \| g(x) - h(x) \|_2 \leq \epsilon \right] \geq 1- \delta.
\end{align*}
\end{definition}

$(\epsilon_{\ap},\delta_{\ap})$-closeness basically requires that, the root-mean-square error of two functions' output with a given input is small enough. 

Now we present our main theoretical result.
\begin{theorem}[Informal of Theorem~\ref{thm:for_general_x}]
\label{thm:main_thm}
For a fully connected neural network (each layer can be viewed as $f(x) = \phi(Ax+b)$), where $\|x\|_2 = 1$ and $x \in \R^d_{\geq 0}$. Applying magnitude-based pruning on the weight $A \in \R^{m \times d}$ (where each $A_{i,j} \sim {\cal N}(0,\sigma_A^2)$) gives us $\wt{A} \in \R^{m \times d}$. 
There exists a function $h(x)$ satisfying two properties :
\begin{enumerate}
    \item $h(x)$ is $(\epsilon_{\dfp},\delta_{\dfp})$-differential privacy on input $x$;
    \item$h(x)$ is $(\epsilon_{\ap},\delta_{\ap})$-close to $g(x) = \phi(\wt{A} x +b)$.
\end{enumerate}

 where 
$
m = \Omega( \poly( 1/ \epsilon_{\ap} , \log(1/\delta_{\ap}) , \log(1/\delta_{\dfp}) ) )
$ 
 and $\sigma_A = O( \epsilon_{\dfp} \delta_{\dfp}/(m^2) )$. 
\end{theorem}
In the above theorem, we should think of $m$ as the width of the neural network, and $d$ as the input data dimension. $\phi$ is the activation function, e.g., $\phi(z) = \max\{z,0\}$.

Regarding the two properties of $h(x)$, property 1 requires $h(x)$ to provide $(\epsilon_{\dfp},\delta_{\dfp})$-differential privacy, and property 2 requires that $h(x)$ is `equivalent' to magnitude-based pruning with the predefined $(\epsilon_{\ap},\delta_{\ap})$-close notation.

\paragraph{Proof Sketch }

We use $\wt{A} \in \R^{m \times d}$ to denote the weight matrix after magnitude-based pruning, and $\ov{A} = \wt{A} - A \in \R^{m \times d}$. We define a noise vector $e \in \R^m$ as follows:
\begin{align*}
    e = \text{Lap} (1,\sigma)^m \circ ( \ov{A} x ).
\end{align*}

The main proof can be split into two parts.
\begin{claim}\label{cla:part1}
 Let $h(x) = f(x) + e \in \R^m$, we can show that $h(x)$ is $(\epsilon_{\dfp}, \delta_{\dfp})$-differential privacy.
\end{claim}
\begin{claim}\label{cla:part2}
 For sufficiently large $m$, we have 
    \begin{align*}
    \Pr \Big[ \frac{1}{\sqrt{m}} \| e - \ov{A} x \|_2 \geq \epsilon_{\ap} \Big] \leq \delta_{\ap}.
    \end{align*}
\end{claim}

To prove Claim~\ref{cla:part1}, we show by the definition of differential privacy, for any inputs $x$ and $y$ with $\|x-y\| \leq 1$, 
$$\Pr_h [ h(x) \in S ] \leq \exp(\epsilon_{\dfp}) \cdot \Pr [ h(y) \in S ] + \delta_{\dfp}.$$ 

To be more specific, we use the fact of $e$ sampled from the Laplace distribution, and bound the ratio $\frac{p_h(h(x) = t\in S)}{p_h(h(y) = t\in S)}$, where $p(\cdot)$ denotes probability density. To bound the above ratio: first we need to derive and upper-bound the global sensitivity (see Appendix~\ref{sec:app_sens}) of a single-layer neural network. Then, we extend the famous anti-concentration result by Carbery and Wright \cite{cw01} to a more general setting (see Appendix~\ref{sec:app_anti}). To the best of our knowledge, this generalization is not known in literature. Once the densities are bounded, integrating $p(\cdot)$ yields the requirement of differential privacy, thus complete the proof of part 1.

To prove Claim~\ref{cla:part2}, we firstly define $z_i = e_i - (\ov{A} x)_i$. Then we apply the concentration theorem (see Appendix~\ref{sec:app_conc}) to show that for any $\|x\|_2 = 1$ and $x \in \R^d_{+}$, 
\begin{align*}
\Pr \Big[ \frac{1}{m} | \sum_{i=1}^{m}(z_i-\E[z_i] ) | \geq \epsilon_{\ap}^2 \Big] \leq \delta_{\ap},
\end{align*}
which completes the proof of Claim~\ref{cla:part2}.


\section{Experiments}
\label{sec:exp}

\begin{figure}[!t]
    \centering
        \includegraphics[width=0.98\textwidth]{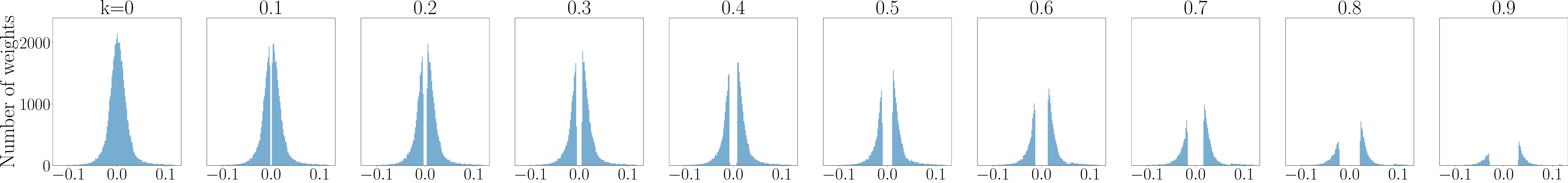}
        \caption{Distribution of weight magnitude (over all LeNet-5 convolution layers) after training and pruning on MNIST.}\label{fig:dist_weight_magnitude_mnist}
\end{figure}

\begin{figure*}
\centering
        \includegraphics[width=0.98\textwidth]{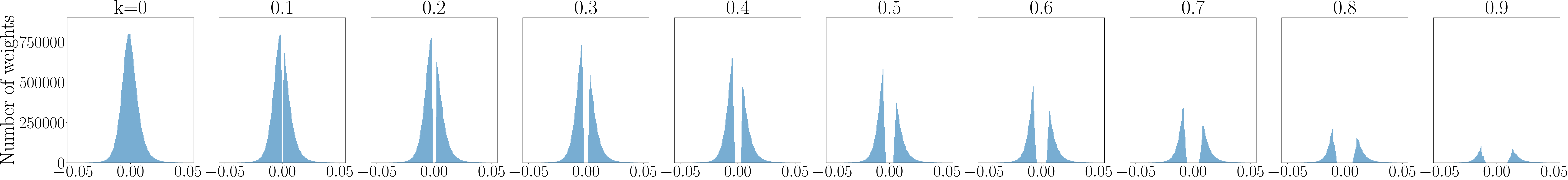}
        \caption{Distribution of weight magnitude (over all VGG-19 convolution layers) after training and pruning on CIFAR-10.}
    \label{fig:dist_weight_magnitude_cifar}
\end{figure*}

\begin{figure}
    \centering
    \includegraphics[width=0.7\linewidth]{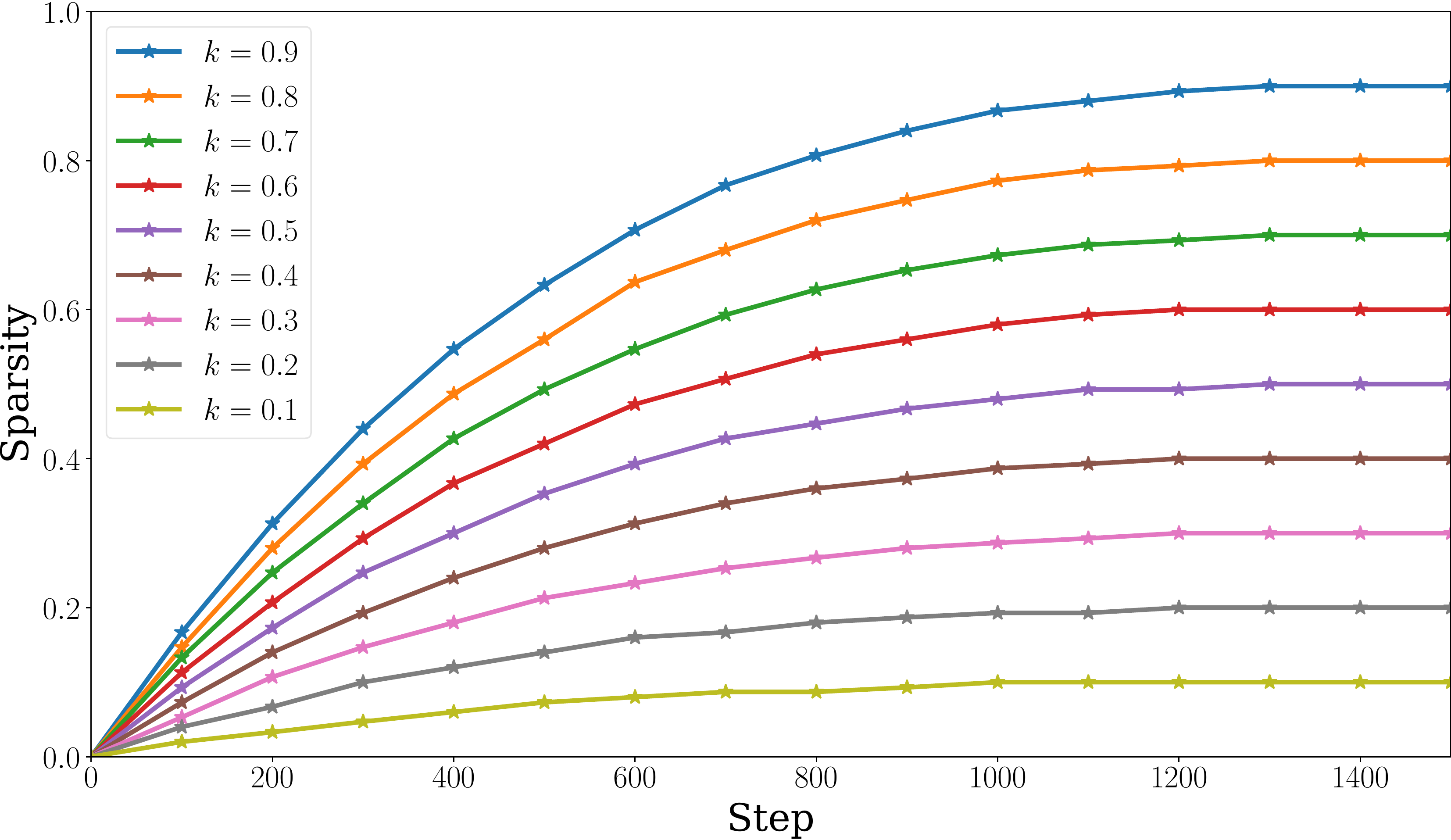}
    \caption{Sparsity-step curve during the pruning stage with gradual pruning mechanism applied.}
    \label{fig:grad_prune}
\end{figure}

This section presents experimental results related to our main theoretical finding in Section \ref{sec:main}. The experiments aim to answer the following questions:

\begin{enumerate}
\item  With the same utility (test accuracy) requirement, does neural network pruning preserve more privacy compared with adding noise? If yes, how much more?

\item  Given a privacy budget $(\epsilon_{\dfp}, \delta_{\dfp})$, how large $m$ should be to guarantee an $(\epsilon_{\ap}, \delta_{\ap})$-closeness between neural network pruning and adding noise?
\end{enumerate}

\subsection{Experimental setup}
\label{sec:dataset_setup}

\paragraph{Datasets and networks.}
We have conducted image classification experiments on MNIST \cite{lcb10} and CIFAR-10 \cite{cifar10} datasets. 

The network architectures used are LeNet-5 \cite{lbbh98} for MNIST, and VGG-19 \cite{sz14} for CIFAR-10. All models are trained on 4 Nvidia Tesla K80 GPUs using Tensorflow \cite{aab+16tensorflow}. We test these two simple and standard architectures in order to better match the setting described in our theoretical result. 
A detailed description of network architectures and hyper-parameters can be found in the Appendix~\ref{sec:app_exp}. 

\paragraph{Pruning.}

Our experiments have employed the gradual pruning technique introduced in \cite{zg17}, where over $n$ iterations starting from $t_0$ with interval $\Delta t$, the sparsity is increased from an initial sparsity value $k_0$ (0 in our case) to a target sparsity value $k_T$ such that,\\ for $t \in\left\{t_{0}, t_{0}+\Delta t, \ldots, t_{0}+n \Delta t\right\}$:
\begin{align*}
k_{t}=k_{T}+\left(k_{0}-k_{T}\right)\left(1- ( t-t_{0} ) / ( n \cdot \Delta t ) \right)^{3} \quad 
\end{align*}

Figure \ref{fig:grad_prune} shows 
sparsity-step curves to illustrate how to achieve a target sparsity. 
Figure \ref{fig:dist_weight_magnitude_mnist} and \ref{fig:dist_weight_magnitude_cifar} show the final weight distribution across all layers after applying gradual magnitude-based pruning.

\subsection{Test of privacy leakage as an inversion attack}
\label{sec:inversion_alg}

Since there is no standard way to quantify privacy preservation, we have used the method that with a given attack, we will measure how much ``privacy leakage" a particular approach will suffer. We have introduced the following pipeline to measure the privacy leakage in the neural network.

\begin{figure*}[!t]
    \centering
    \subfloat[Test accuracy of LeNet-5 trained on MNIST with different sparsity levels.]{\includegraphics[width=0.48\linewidth]{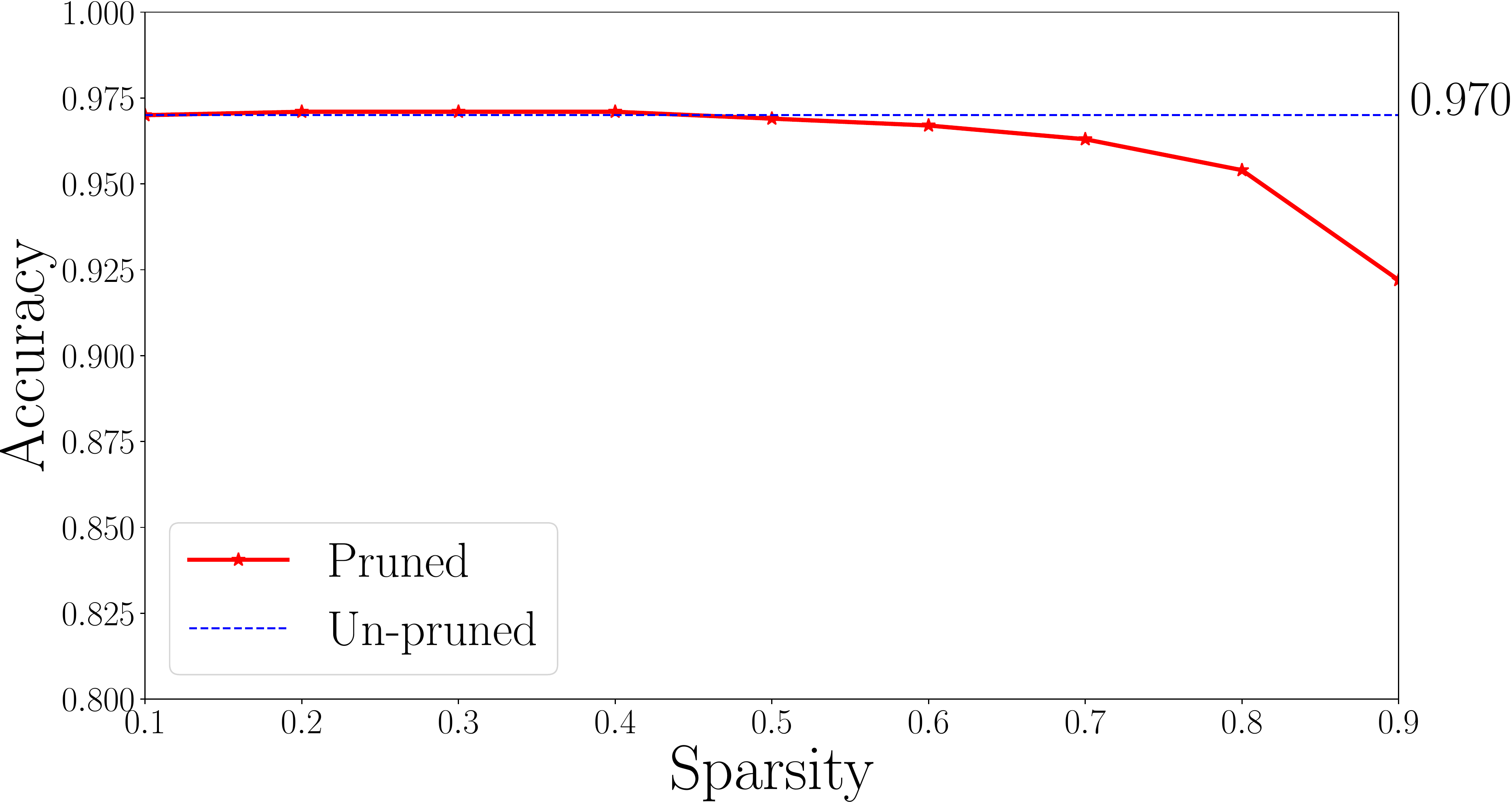}}\hspace{2mm}
    \subfloat[Test accuracy of VGG-19 trained on CIFAR-10 with different sparsity levels.]{\includegraphics[width=0.48\linewidth]{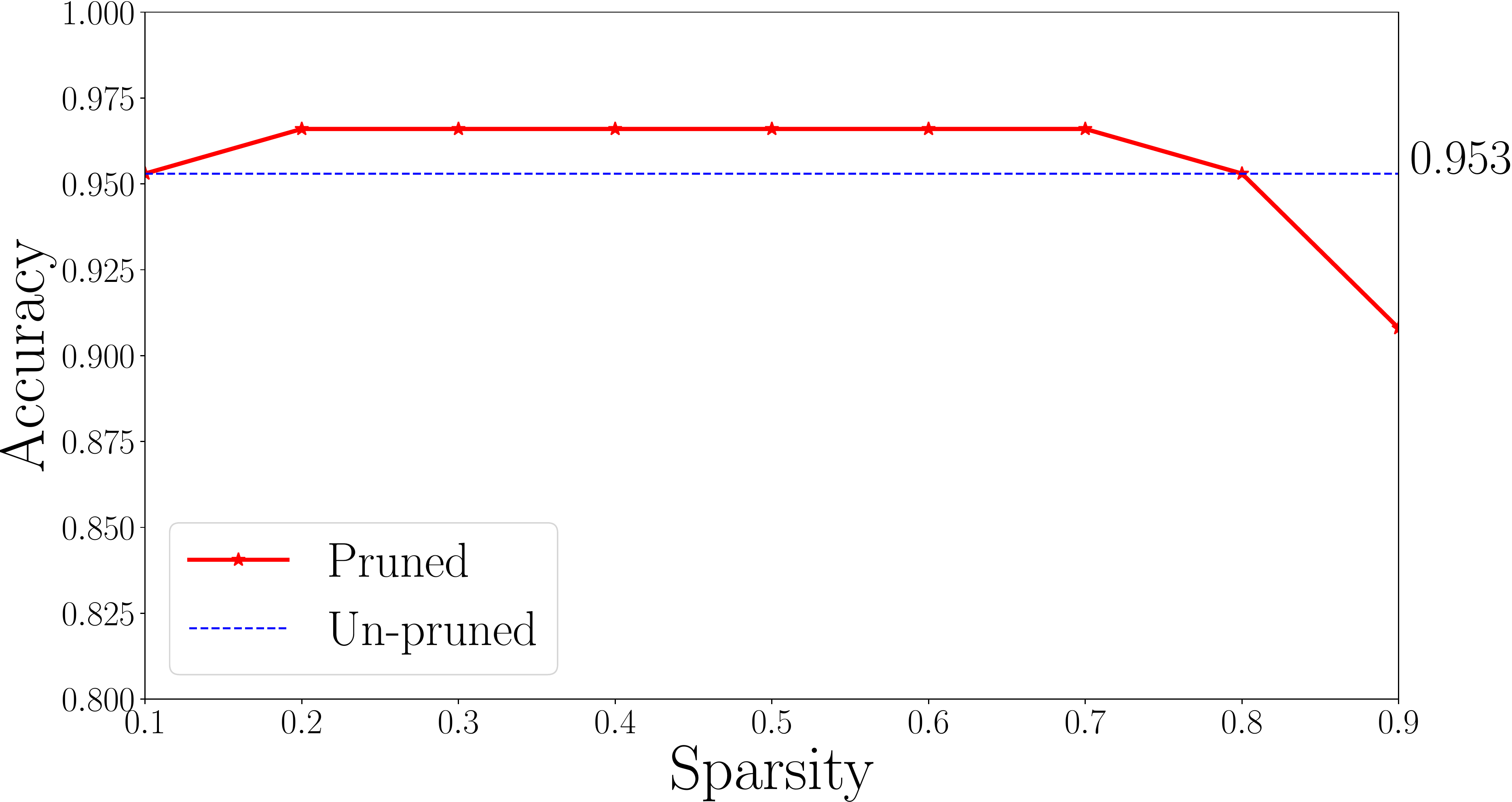}}
    \caption{Demonstration of test accuracy of LeNet-5 on MNIST (a), and test accuracy of VGG-19 on CIFAR-10 (b).}
    \label{fig:grad_and_acc_main}
\end{figure*}

Let us conceptualize the mapping from the input to the hidden-layer output in a neural network as a representation function. For a public representation function $\Phi:\R^{d} \rightarrow \R^{m}$, we perturb $\Phi$ to get $\Phi':\R^{d} \rightarrow \R^{m}$, and keep $\Phi'$ private. We use the following attack \cite{mv15} to test the potential privacy leakage of using $\Phi'$ as the representation function on any input $x$:

Given a perturbed representation $\Phi'_0 = \Phi'(x)$, and the public function $\Phi$, the attacker's goal is to find the preimage of $\Phi'(x)$, namely
\begin{align*}
    x^* = \arg\min_{x\in \R^{d}} {\cal L} ( \Phi(x) , \Phi'_0 ) + \lambda {\cal R}(x)
\end{align*}
where the loss function ${\cal L}$ is defined as 
\begin{align*}
    {\cal L} ( a , a' ) = \| a - a' \|_2^2,
\end{align*}
$\lambda > 0$ is the regularization parameter, and the regularization function $\mathcal{R}$ in our case is the total variation of a 2D signal
\begin{align*}
    \mathcal{R}(a) = \sum_{i,j} ( (a_{i+1, j} - a_{i,j})^2 + (a_{i, j+1} - a_{i,j})^2 )^{1/2}.
\end{align*}

Answering Question 1 requires the testing of the inversion attack against two different perturbations $\Phi'$:
\begin{itemize}
    \item $\Phi'_\text{noise}(x) = \Phi(x) + e$, where $e$ is a Laplace noise
    \item $\Phi'_\text{prune}(x) = \tilde{\Phi}(x)$, where $\tilde{\Phi}$ stand for magnitude-based pruning on layers of $\Phi$
\end{itemize}

The following describes how to quantitatively measure the leakage with the defined inversion attack.

\paragraph{Measurements of privacy leakage under attack}
Let us denote the preimage of $\Phi'(x)$ obtained from the inversion attack as $x^*$, and the original image as $x$. If $\Phi'$ has a strong differential privacy guarantee, namely $\epsilon_{\dfp}$ is sufficiently small, then the attacker's chance of finding $x$ as the preimage of $\Phi'(x)$ will only be marginally better than random guessing. This observation motivates us to use the closeness between $x*$ and $x$ as an indicator for privacy leakage. In experiments, we adopt the following four metrics to measure the similarity, or closeness between $x^*$ and $x$:

\begin{itemize}
    \item Normalized structural similarity index metric (SSIM). SSIM is a perception-based metric that considers the similarity between images in structural information, luminance and contrast. The detailed calculation can be found in \cite{wbs+04}. We normalize SSIM to take value range $[0,1]$ (original SSIM takes value range $[-1, 1]$).
    \item SIFT similarity (SIFT). We first calculate SIFT \cite{l04} keypoints and descriptors ${\cal K}^*$ and ${\cal K}$ based on $x^*$ and $x$. Then, we search for matched pairs  between ${\cal K}^*$ and ${\cal K}$, and compute the distance of each matched pair. We then filter the good matched pairs based on the distance, and calculate $\frac{\text{\#~matched~good pairs}}{\text{\#~matched~pairs}} $. 
    \item Normalized complementary pHash distance (HASH). We first use a 64-bit perceptual hashing \cite{z10} function $H(\cdot)$ to get hashes for $x^*$ and $x$. Then we calculate the Hamming distance $d(\cdot)$ between $H(x^*)$ and $H(x)$. The normalized complementary pHash distance is then defined to be $(64 - d_{\mathrm{ham}} ( H( x^* ) , H( x ) ))/64$. This metric captures the global similarity between $x^*$ and $x$.
    \item Inference accuracy (INFE). We run the inference of $f(W,x^*)$, and then calculate the accuracy: $\mathbf{1}_{y=f(W,x^*)}$, where $y$ is the label of $x$. This metric measures whether $x^*$ contains the `class' information of $x$.
\end{itemize}

For all four metrics above, a larger value indicates a higher similarity between $x^*$ and $x$.

\subsection{Magnitude-based pruning vs. differential privacy}
\label{sec:exp_results}
Now we report our experimental results that strongly suggest that magnitude-based pruning preserves more privacy (suffer less leakage) than adding differentially private noise.

\paragraph{Utility of pruned networks.}
Figure \ref{fig:grad_and_acc_main}(a) and (b) show the accuracy results on MNIST and CIFAR-10 of corresponding networks pruned by the gradual magnitude-based pruning algorithm at different sparsity levels. 

The task on MNIST achieves the same level of accuracy as the network model without pruning when the sparsity level is $\le 0.5$ ($0.5$ means $50\%$ of the weights in a network are zeros).  Its accuracy gradually decreases as the sparsity level increases.  The test with CIFAR-10 maintains the same accuracy as or better than the model without pruning when the sparsity level is $\le 0.8$, and then gradually decreases.

Note that our accuracy results at higher sparsity levels are lower than those reported in \cite{hmd15}. There are two hypotheses.  First,  our experiment used a gradual pruning algorithm for speed which may introduce some accuracy losses.  Second, as suggested by \cite{geh19}, magnitude-based pruning with different layer-wise sparsity (as \cite{hmd15} did) yields better accuracy than pruning layers with the same target sparsity as in our case. 

However, for answering the question if network pruning achieves better utility than differential privacy for the same targeted accuracy,  our experiments can be viewed as conservative results.  

\begin{figure*}[!t]
    \centering
    \subfloat[SIFT similarity]{\includegraphics[width=0.24\linewidth]{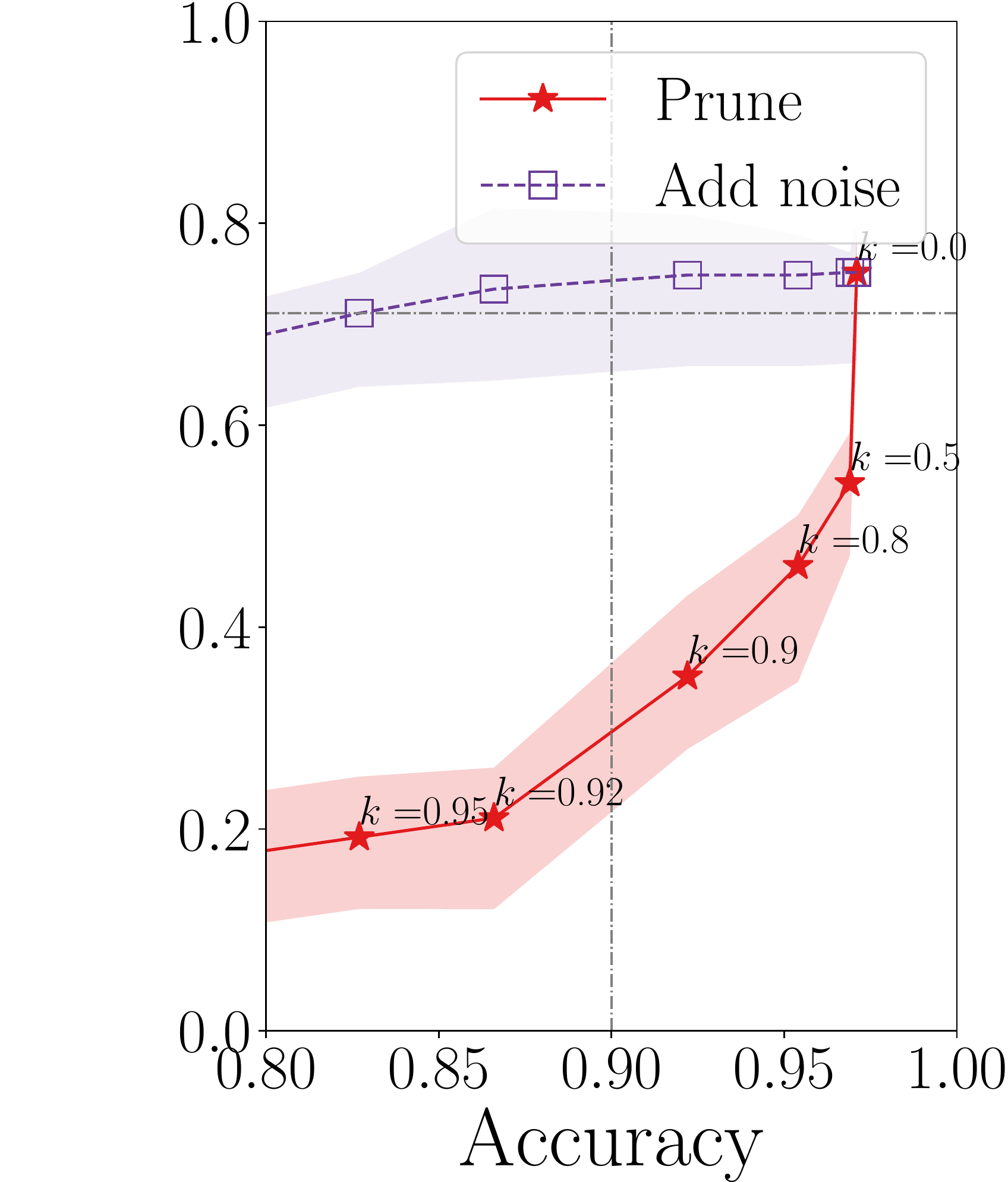}}
    \subfloat[SSIM similarity]{\includegraphics[width=0.24\linewidth]{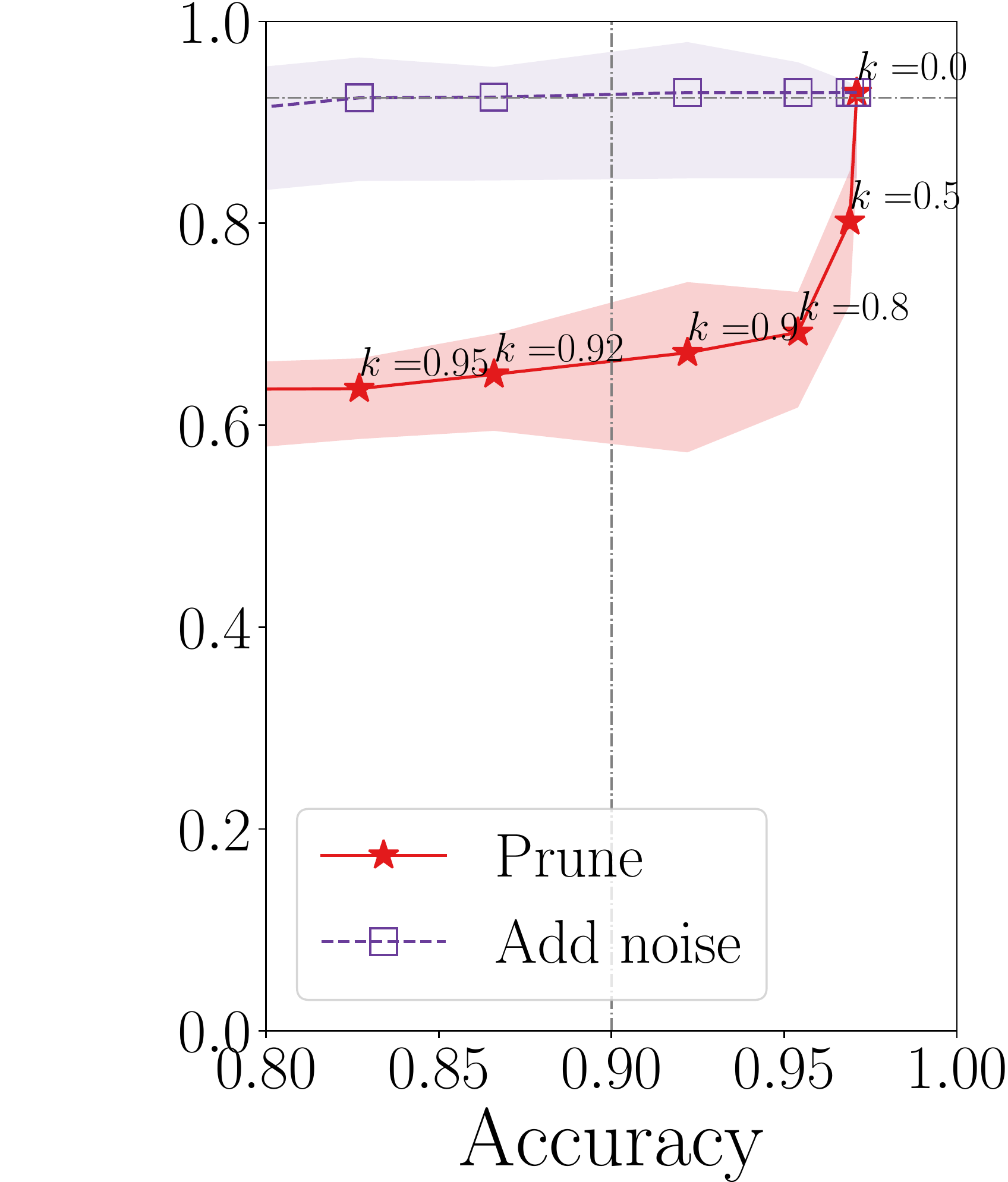}}
    \subfloat[HASH similarity]{\includegraphics[width=0.24\linewidth]{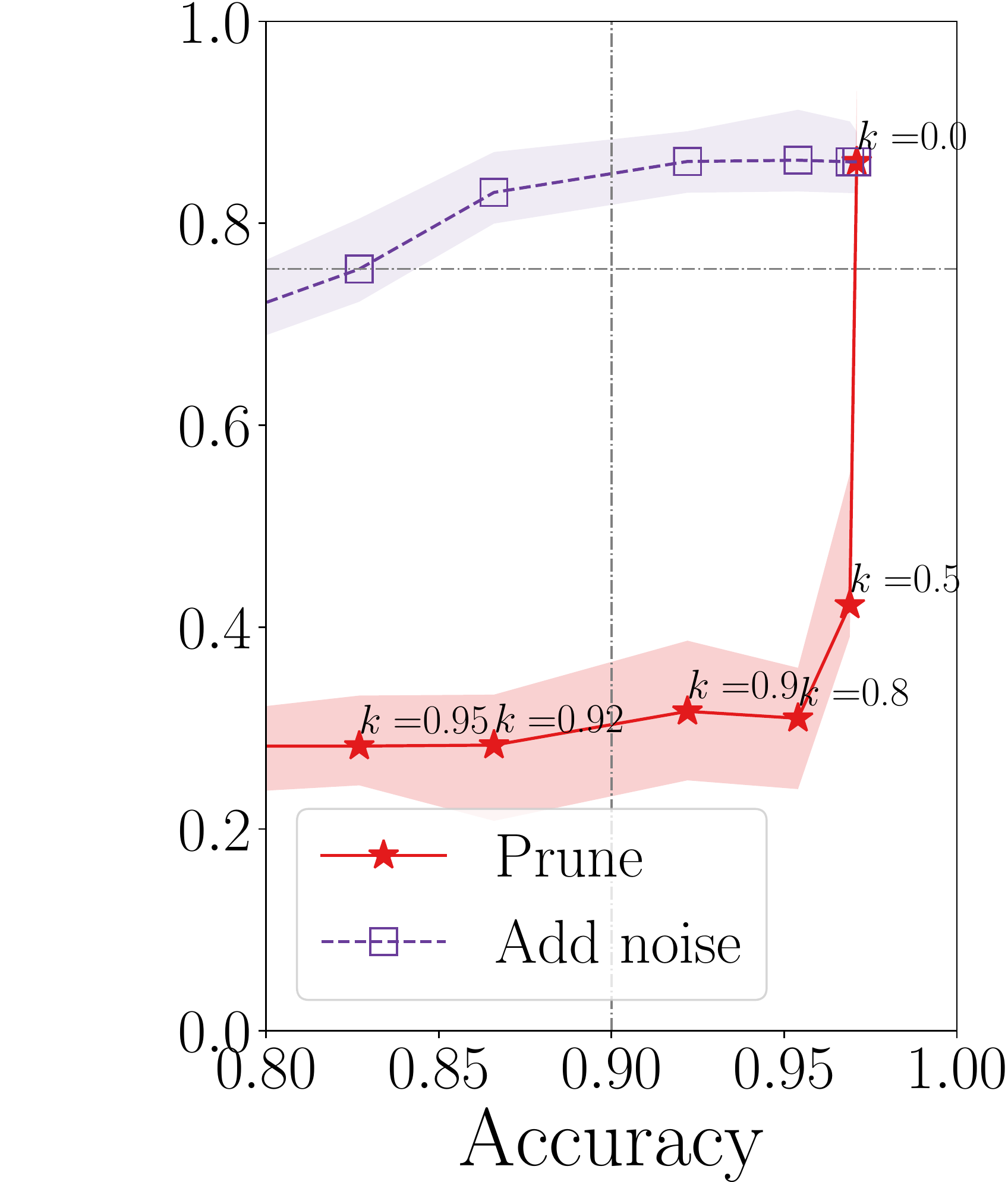}}
    \subfloat[INFE similarity]{\includegraphics[width=0.24\linewidth]{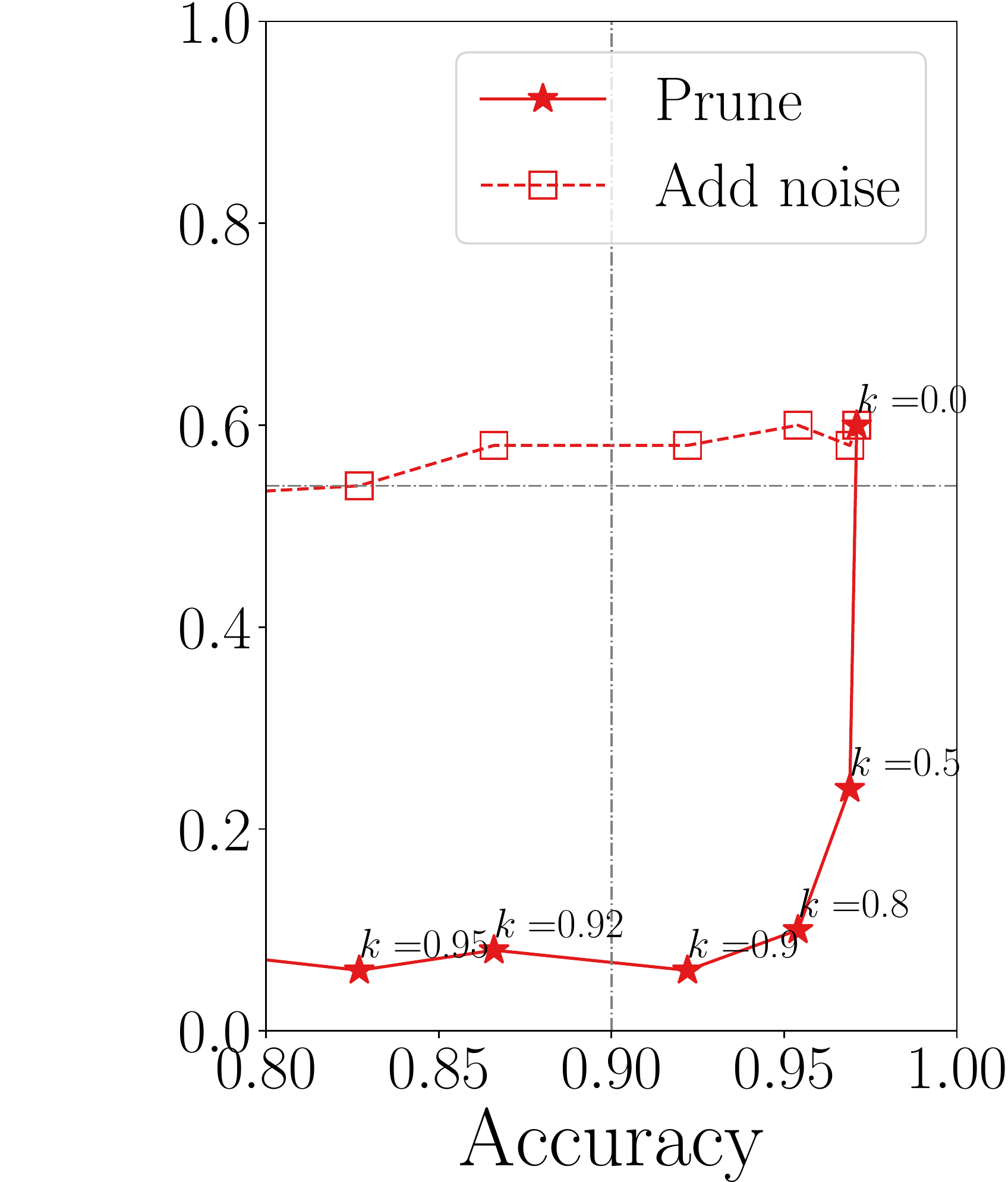}}\\
    \subfloat[Inverted MNIST digit from Conv5 layer with pruning (1st row) or adding noise (2nd row)]{\includegraphics[width=0.8\linewidth]{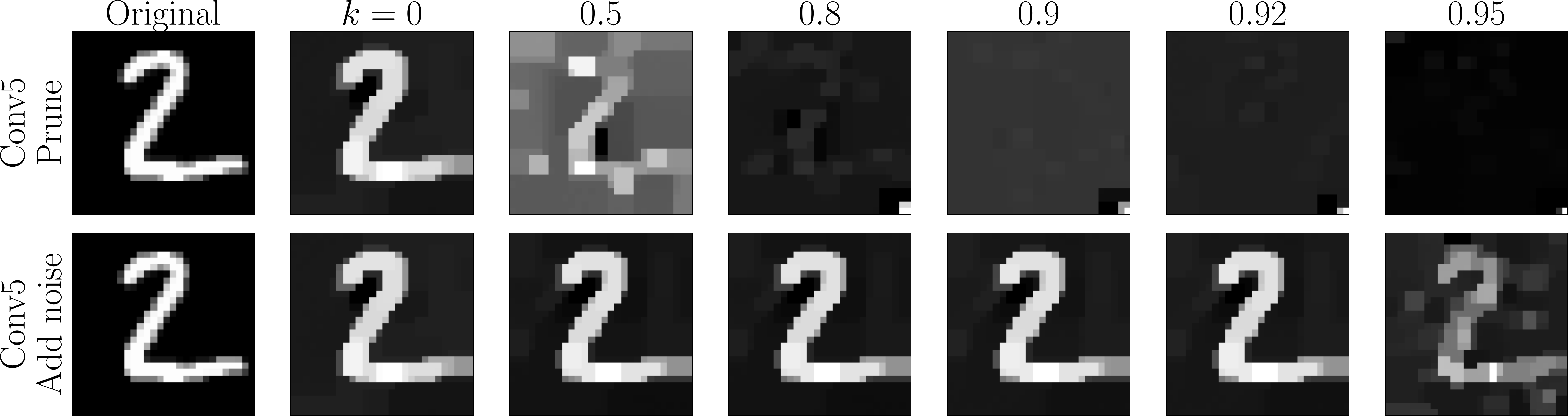}}
    \caption{Privacy leakage of pruning and adding noise measured by similarity between inversion and original images (a) - (d), obtained from Conv5 layer from LeNet-5 trained on MNIST. Shadow represents the value range (we didn't report the value range for INFE, because it always has 0 as minimum and 1 as maximum). Vertical dashed lines match the accuracy; horizontal dashed lines match the privacy leakage (similarity between inverted and original images). Also visualization of inverted digits from layers with pruning or adding noise (e). Sparsity ($k$) is annotated.}
    \label{fig:sim_mini_mnist}
\end{figure*}

\begin{figure*}[!t]
    \centering
    \subfloat[SIFT similarity]{\includegraphics[width=0.24\linewidth]{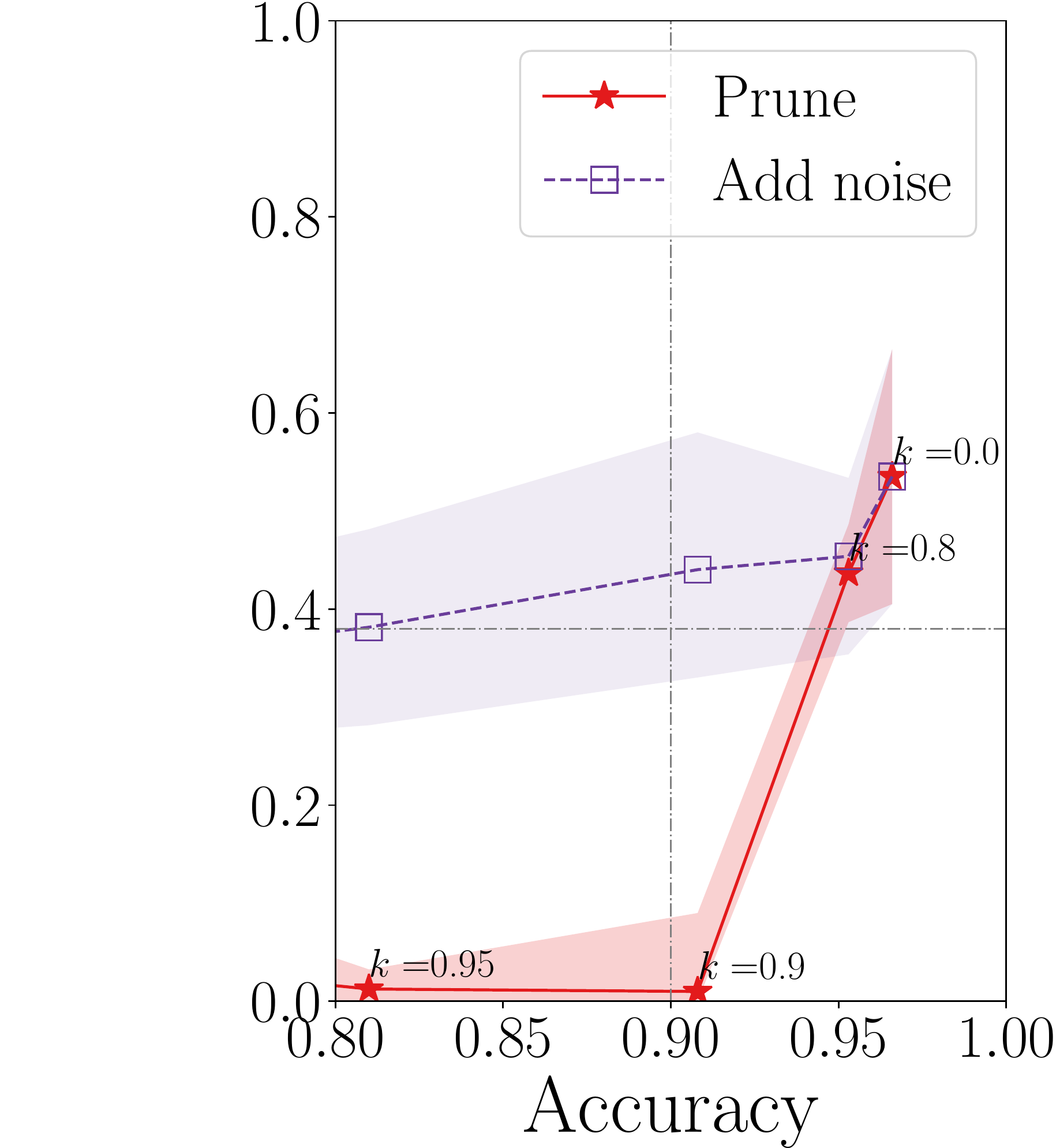}}
    \subfloat[SSIM similarity]{\includegraphics[width=0.24\linewidth]{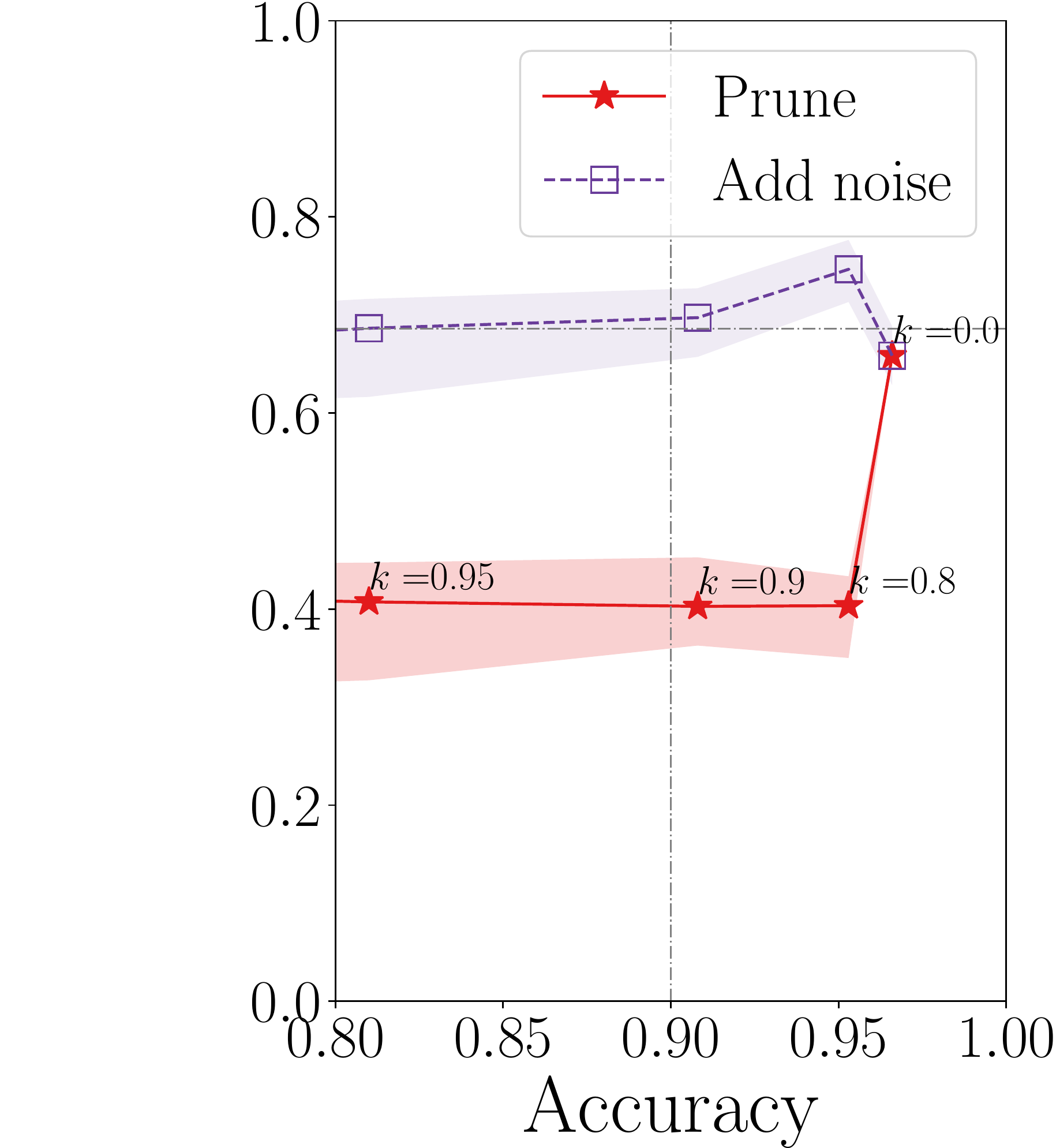}}
    \subfloat[HASH similarity]{\includegraphics[width=0.24\linewidth]{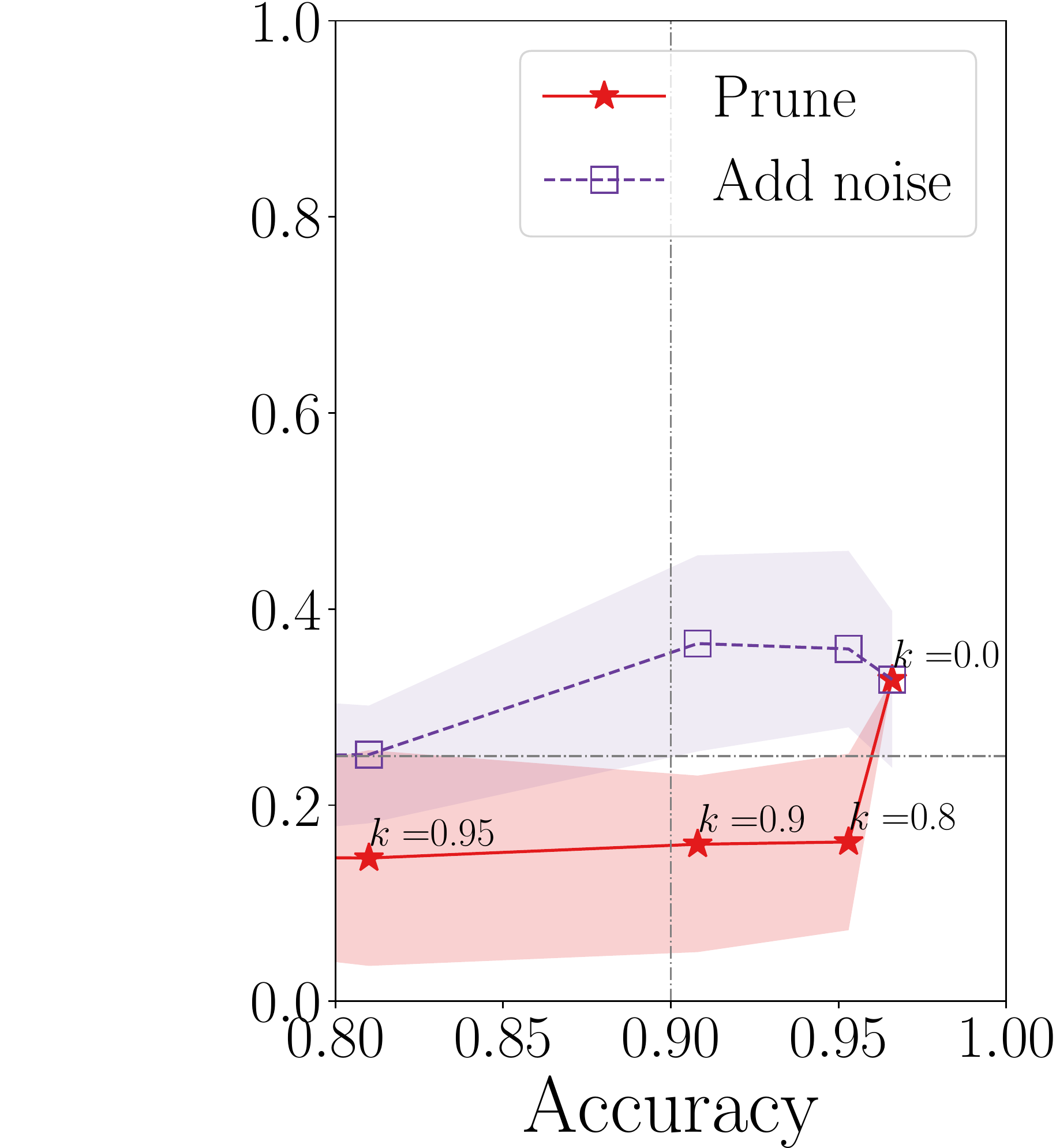}}
    \subfloat[INFE similarity]{\includegraphics[width=0.24\linewidth]{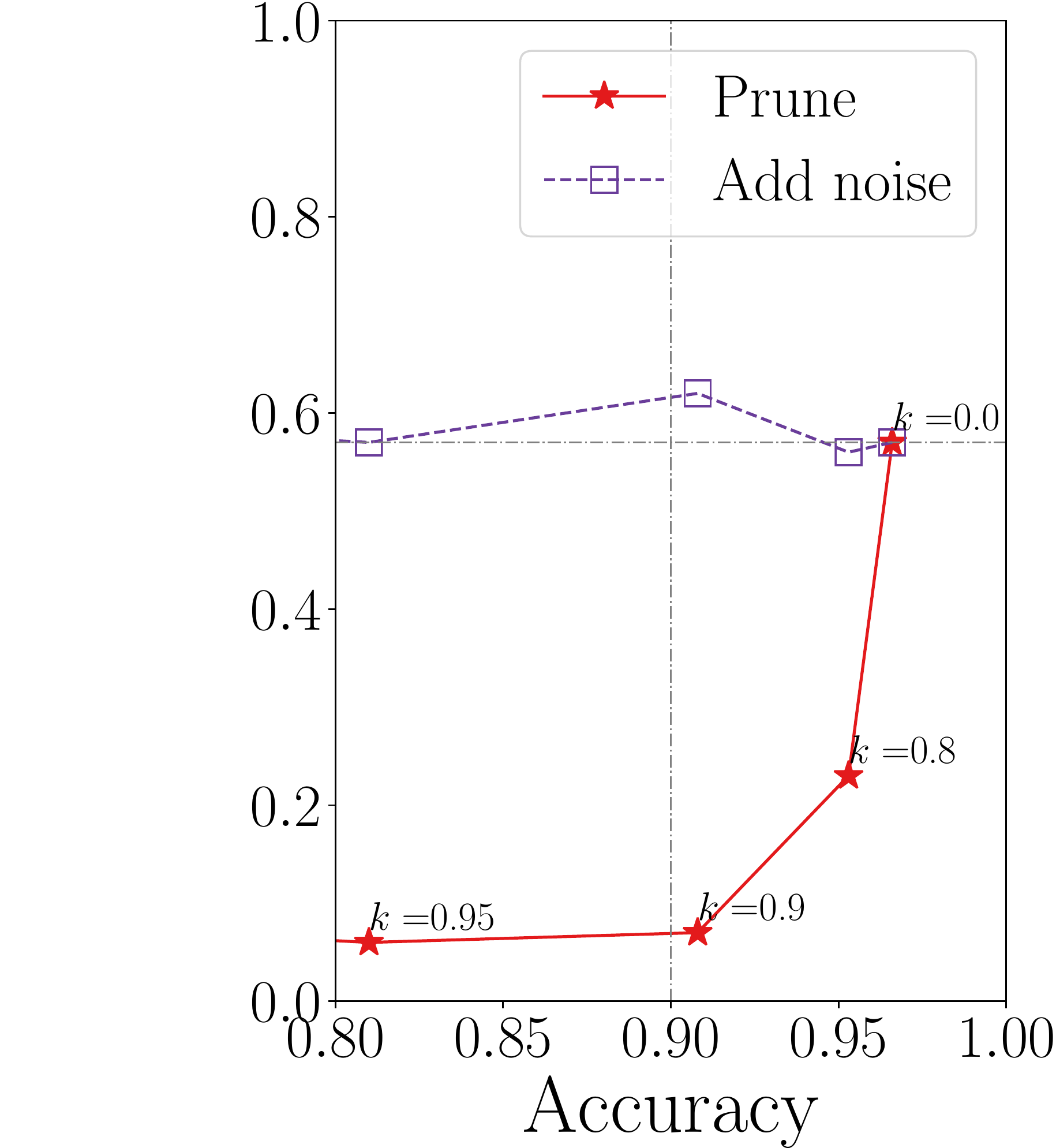}}\\
    \subfloat[Inverted CIFAR-10 image from Conv5-1 layer with pruning (1st row) or adding noise (2nd row)]{\includegraphics[width=0.8\linewidth]{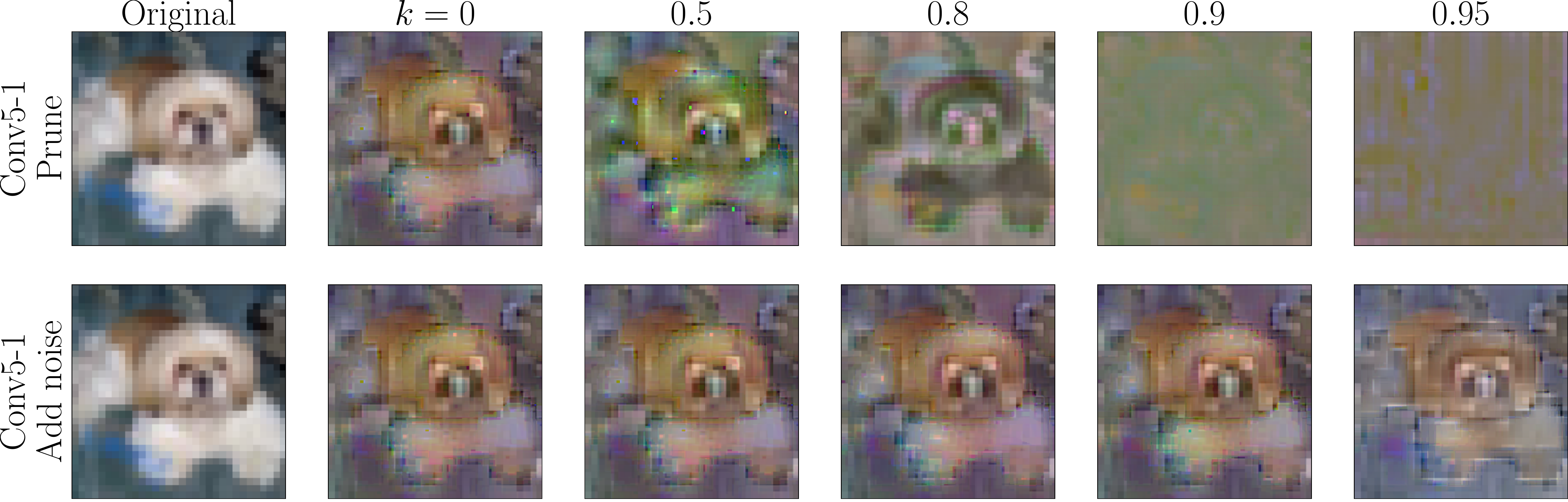}}
    \caption{Privacy leakage of pruning and adding noise measured by similarity between inversion and original images (a) - (d), obtained from Conv5-1 layer from VGG-19 trained on CIFAR-10. Shadow represents the value range (we didn't report the value range for INFE, because it always has 0 as minimum and 1 as maximum). Vertical dashed lines match the accuracy; horizontal dashed lines match the privacy leakage (similarity between inverted and original images). Also visualization of inverted digits from layers with pruning or noise (e). Sparsity ($k$) is annotated.}
    \label{fig:sim_mini_cifar}
\end{figure*}

\begin{figure}[t]
    \centering
    \includegraphics[width=0.8\linewidth]{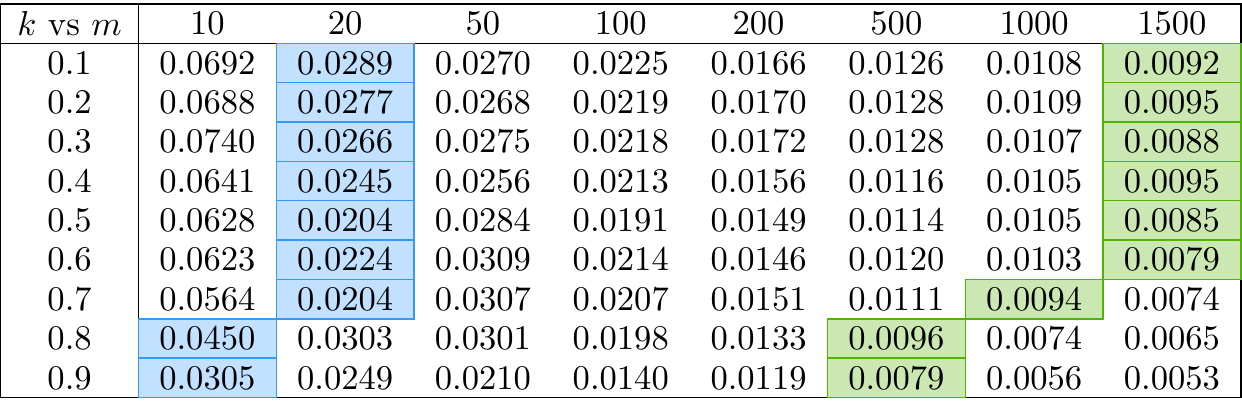}
    \caption{Error ($\frac{1}{\sqrt{m}}\|h(x)-g(x)\|_2$) with different sparsity $k$ and neural network width $m$. For each $k$ (each row), the value with \textcolor{b2}{blue} background indicates the first time that error goes below 0.05 when we increase $m$, and the ones with \textcolor{mygreen}{green} background is for the first time that error goes below 0.01. Best viewed in color.}
\label{fig:test_m}
\end{figure}

\paragraph{Relative privacy leakage.}
To answer Question 1, we need to compare the privacy leakage of applying magnitude-based pruning ($\Phi'_\text{prune}$) and adding differentially private noise ($\Phi'_\text{noise}$) while requiring them to yield the same test accuracy. Specifically, we tune the amount of noise in $\Phi'_\text{noise}$ to make applying $\Phi'_\text{prune}$ and  $\Phi'_\text{noise}$ yield the same test accuracy. We then run the inversion attack defined in Section \ref{sec:inversion_alg} on $\Phi'_\text{prune}$ and $\Phi'_\text{noise}$, and measure the closeness between the inverted sample $x^*$ and the original sample $x$.

Figure~\ref{fig:sim_mini_mnist} and Figure~\ref{fig:sim_mini_cifar} show the privacy leakage differences between magnitude-based pruning and adding differentially private noise for the experiments on MNIST and CIFAR-10 respectively.  Given a certain layer, we have executed the inversion attack on 100 randomly chosen images, and calculated the average score of each metric between the original images and the inverted ones. We also report the value range. The sub-figures (a), (b), (c), and (d) of both figures show the similarity curves of the four metrics for conv5 layer in the MNIST and CIFAR-10 experiments. The full version with multiple layers can be found in Appendix~\ref{sec:app_exp}.

These results show that for the same accuracy, there are substantial gaps between magnitude-based pruning 
and adding differential private noise in terms of the similarity measures between inversion and original images,
which suggest that neural network pruning preserve much more privacy than adding noise.

Take CIFAR-10 (see Figure \ref{fig:sim_mini_cifar}) as an example. When adding noise and pruning achieve the same similarity (drawing vertical lines on the graph), the privacy leakage (measured as similarity) of $\Phi'_\text{prune}$ is always smaller than that for $\Phi'_\text{noise}$, which suggests that with the same utility requirement, pruning helps preserve more privacy than adding noise. This observation gives us a clear `yes' to Question 1. If viewed from a different perspective, say matching the privacy leakage instead of accuracy, which can be interpreted as drawing horizontal lines on the graph, pruning always yields a higher accuracy ($\sim  10\%$ as shown in Figure \ref{fig:sim_mini_cifar}) than adding differentially private noise.

To visually demonstrate the gap between the two approaches, Figure~\ref{fig:sim_mini_mnist}(e)  and \ref{fig:sim_mini_cifar}(e) show the inversions from $\Phi'_\text{prune}$ and  $\Phi'_\text{noise}$  with their accuracy matched. For LeNet-5 trained on MNIST (see Figure \ref{fig:sim_mini_mnist}(e)), inverted digits from $\Phi'_\text{noise}$ still leak much structural information of the original image, while inversion from $\Phi'_\text{prune}$ are more vague (see sparsity smaller than 0.5), or even unidentifiable (see sparsity larger than 0.9). Inverted images from VGG-19 layers trained on CIFAR-10 share consistent observations. This implies that with the same accuracy, pruning may preserve more privacy than adding noise.

\subsection{Width vs accuracy}
\label{sec:exp_q2}
In this section, we aim to answer Question 2 by testing in a synthetic setting. We show that an $m$ of order $10^3$ may be sufficient to give the equivalence in Theorem \ref{thm:main_thm} with a realistic privacy budget.

Theorem \ref{thm:main_thm} indicates that with a sufficiently large $m$, we can always draw an equivalence between adding differentially private noise and applying magnitude-based pruning. However, the theorem is based on the {\em worst-case} analysis, and does not provide implications for average cases. Also, it is usually infeasible to have a large $m$ (say, larger than $10^5$) in real-world applications. In order to bridge the gap between the theory and the practice, we reproduce the setup of Theorem \ref{thm:main_thm}, and test how large $m$ needs to be for a synthetic setting. 

Specifically, our test is based on a single layer neural network denoted as $f(x) = \phi(Ax+b)$, where $A \sim {\cal N}(0, \sigma_A)^{m\times d}$ and $\phi(a) = \max(a,0)$ is the ReLU activation function. We generate $x$ from the folded Gaussian distribution, and normalize them such that $\|x\|_2 = 1$. For simplicity of test, we set the bias term $b$ to be 0.

Our goal is to figure out that, given a target sparsity $k$, how large $m$ needs to be to guarantee an $(\epsilon_{\ap}, \delta_{\ap})$-closeness (see Definition \ref{def:epsilon_delta_close}) between the following two perturbations on $f(x)$:

\begin{itemize}
    \item Adding $(\epsilon_{\dfp},\delta_{\dfp})$-differentially private noise $e$ to $f(x)$:  $h(x) = f(x) + e $
    \item Applying magnitude-based pruning on $A$ and get $\tilde{A}$: $g(x) = \phi(\tilde{A}x+b)$
\end{itemize}
where $e = \mathrm{Lap}(1,\sigma)^m\circ (\tilde{A}-A)x $.

Usually, an $\epsilon_{\dfp}\leq 1$ is sufficient for privacy-sensitive applications. Thus in the test, we set the privacy budgets to be $\epsilon_{\dfp}=0.5, \delta_{\dfp}=0.1$ for simplicity. We run tests for 9 sparsity levels (from $0.1$ to $0.9$ with interval $0.1$), and vary $m$ to measure the closeness between $h(x)$ and $g(x)$, namely $\epsilon_{\ap} = \frac{1}{\sqrt{m}}\|h(x)-g(x)\|_2 $. The results are shown in Figure \ref{fig:test_m}. As can be seen, for all sparsity levels, $m = 20$ is sufficient to satisfy an $\epsilon_{\ap} = 0.05$. To guarantee a stronger closeness of $\epsilon_{\ap} = 0.01$, we can increase $m$ accordingly to 500 for sparsity levels higher than 0.8, and to 1500 for lower ones. This suggests that in practice, $m$ of order $10^3$ will be sufficient given the equivalence in Theorem \ref{thm:main_thm}.


\section{Related work}\label{sec:related}

\paragraph{Pruning algorithms for neural network} 
Traditional deep neural network models are computationally expensive and memory intensive, which hinders their deployment in applications with limited memory resources or strict latency requirements.  Many progress has been made to perform model compression in deep networks, including low-rank factorization \cite{sks+13, lgr+14}, network pruning  \cite{lds90,sb15,hmd15,lkd+16}, and knowledge distillation \cite{hvd15,ccy+17}. Among them, neural network pruning  has been widely adopted because it is able to reduce model sizes by up to one order of magnitude without significant accuracy loss. The idea of network pruning dates back to the Optimal Brain Damage in 1990s \cite{lds90}. Recently, it has been shown that removing the weights with low magnitude can also achieve a highly compressed model \cite{hmd15}, which is referred to as `magnitude-based pruning'. A recent work by Zhu, Liu, and Han \cite{zlh19} empirically observed that pruning neural network gradients helps alleviate privacy leakage without much utility loss. However, they did not provide theoretical explanation.

\paragraph{Differential privacy for deep learning}

The concept of $\epsilon$-differential privacy was originally introduced by Dwork, McSherry, Nissim and Smith \cite{dmns06}. Later, it was generalized to a  relaxation of $(\epsilon,\delta)$-differential privacy \cite{dkmmn06, d09, dr14}. Differential privacy has been successfully applied to many problems. For more detailed surveys of the applications of differential privacy, we refer the readers to \cite{d08,d11b}. 

Applying differential privacy techniques in deep learning is an interesting but non-trivial task. Previous research have customized differential privacy for different learning tasks and settings \cite{ss15, acg+16, pwwd16}.

However, most of these approaches still use the standard mechanism of adding noise to satisfy differential privacy, while this work aims to draw an interesting connection between differential privacy and neural network pruning.

\paragraph{Neural network inversion and sparse recovery}

The problem of inverting a neural network is to find the input data point $x \in \R^d$ that yields a given hidden-layer output $y \in \R^m$ under the assumption that weight matrix $A \in \R^{m \times d}$ is known.  Approaches for neural network inversion generally fall into two categories: by sampling and back-propagation \cite{lk94,jrm+99,lkn99,vr05,mv15}, and by learning a decoder network to invert \cite{b95,db16}. Arora, Liang and Ma \cite{alm15} suggests from a theoretical perspective that it is possible to construct and train a generative model which is the reverse of the feedforward network. \cite{ljdg19} studied how to invert deep generative model. More specifically, they show that for some realizable case, single layer inversion can be performed exactly in polynomial, by solving a linear program. Further, they show that for multiple layers, inversion is NP-hard.

The task of inverting a single-layer linear neural network is intrinsically similar to the classic notion of sparse recovery \cite{crt06b,d06,glps10,hikp12a,ik14,k16,k17,ns19,nsw19},  which aims to reconstruct an approximately $k$-sparse vector $x \in \R^d$ from linear measurements $y = A x$, where $A \in \R^{m \times d}$. Note that in sparse recovery context, sparsity denotes the number of non-zero entries, which is the opposite to what people use in the pruning community. To connect neural network inversion with sparser recovery, one can think of the single-layer linear network's weight matrix as $A$, the hidden-layer output as $y$, and the network input as $x$. In the linear case, it is known that in order to achieve the sparse recovery task, the column sparsity of the sensing matrix $A$ has non-trivial lower bound \cite{nn13}. That is to say, for a single-layer linear neural network, given the hidden-layer output $y=Ax$, we need $A$ to have at least some fraction of non-zero to recover the input $x$. This naturally motivates us to think that pruning may helps preserve privacy.


\section{Conclusions}\label{sec:conclusion}

This paper has presented a theoretical result to show that, if a fully-connected layer of a neural network is wide enough, magnitude-based neural network pruning is equivalent to adding differentially private noise. 

To understand the gap between the theory and practice, the paper reports experimental results, in a synthetic setting with a realistic differential privacy budget, that
the width of the neural network needs to be only a few hundreds in order to make the theoretical equivalence hold.  

Our experiments with MNIST and CIFAR-10 show that, with the same target accuracy, magnitude pruning preserves more privacy than adding random noise (the classical way to provide differential privacy) to neural network. 

These results have strong practical implications for two reasons.  First, since neural network pruning has the property that sparsity can be quite high (e.g. $>90$\%) without reducing inference accuracy, it strongly suggests that network pruning can be an effective method to achieve differential privacy without any or much reduction of accuracy.  Second, although the result is for a single layer of a neural network, it is quite natural in a distributed or federated learning system to use a particular layer to communicate among multiple sites.

Several questions remain open. First, Theorem \ref{thm:main_thm} is only for a single-layer fully connected network, and it would be interesting if one can extend it to multi-layer settings and also convolutional neural networks. 

Second, our theoretical finding is based on the worst case analysis, which means in most cases, $m$ can be much smaller. How to efficiently determine $m$ for different settings requires more investigation. 

Third, this paper has used several similarity measures as metrics for privacy leakage in the absence of true privacy leakage measure.  How to quantify privacy leakage is a challenging question.

Finally, in order to use network pruning as a mechanism to preserve privacy in a practical distributed or federated learning system, one needs to consider many design details including which layers to prune, whether or not to prune layers with the same sparsity, where the work of pruning should be performed, and how to coordinate among multiple sites.   

\section*{Acknowledgments}

This project is funded in part by Princeton University fellowship, Pony Ma Foundation, Simons Foundation, Schmidt Foundation, NSF, DARPA/SRC, Google and Amazon.


The authors would like to thank Paul Beame, Xin Yang and Ruizhe Zhang for very useful discussions about anti-concentration section. The authors would like to thank Tianren Liu for very useful discussions about cryptography. The authors would like to thank Inrit Dinur, Xiangru Jian, Gautam Kamath, Adam Klivans, Zhiyuan Li, Binghui Peng, Nikunj Saunchi, Zhenyu Song, Daniel Suo, Grant Wallace and Fan Yi for very useful discussions.

\newpage
{
\bibliographystyle{alpha}
\bibliography{ref}
}

\onecolumn
\appendix

\section{Probability tools}\label{sec:app_prob}

In this section we present a number of classical probability tools used in the proof.
Lemma \ref{lem:chernoff} (Chernoff), \ref{lem:hoeffding} (Hoeffding) and \ref{lem:bernstein} (Bernstein)  are about tail bounds for random scalar variables.
Lemma \ref{lem:anti_gaussian} and Lemma~\ref{lem:concen_gaussian} state two standard results for random Gaussian variable. Lemma~\ref{lem:chi_square_tail} is a probability for Chi-square distribution.
Finally,
Lemma \ref{lem:matrix_bernstein} is a concentration result on random matrices.

We state the classical Chernoff bound which is named after Herman Chernoff but due to Herman Rubin. It gives exponentially decreasing bounds on tail distributions of sums of independent random variables.
\begin{lemma}[Chernoff bound \cite{c52}]\label{lem:chernoff}
Let $X = \sum_{i=1}^n X_i$, where $X_i=1$ with probability $p_i$ and $X_i = 0$ with probability $1-p_i$, and all $X_i$ are independent. Let $\mu = \E[X] = \sum_{i=1}^n p_i$. Then \\
1. $ \Pr[ X \geq (1+\delta) \mu ] \leq \exp ( - \delta^2 \mu / 3 ) $, $\forall \delta > 0$ ; \\
2. $ \Pr[ X \leq (1-\delta) \mu ] \leq \exp ( - \delta^2 \mu / 2 ) $, $\forall 0 < \delta < 1$. 
\end{lemma}

We state the Hoeffding bound:
\begin{lemma}[Hoeffding bound \cite{h63}]\label{lem:hoeffding}
Let $X_1, \cdots, X_n$ denote $n$ independent bounded variables in $[a_i,b_i]$. Let $X= \sum_{i=1}^n X_i$, then we have
\begin{align*}
\Pr[ | X - \E[X] | \geq t ] \leq 2\exp \left( - \frac{2t^2}{ \sum_{i=1}^n (b_i - a_i)^2 } \right).
\end{align*}
\end{lemma}

We state the Bernstein inequality:
\begin{lemma}[Bernstein inequality \cite{b24}]\label{lem:bernstein}
Let $X_1, \cdots, X_n$ be independent zero-mean random variables. Suppose that $|X_i| \leq M$ almost surely, for all $i$. Then, for all positive $t$,
\begin{align*}
\Pr \left[ \sum_{i=1}^n X_i > t \right] \leq \exp \left( - \frac{ t^2/2 }{ \sum_{j=1}^n \E[X_j^2]  + M t /3 } \right).
\end{align*}
\end{lemma}

We state two bounds for Gaussian random variable:
\begin{lemma}[folklore]\label{lem:concen_gaussian}
    Let $X \sim {\cal N}(0,\sigma^2)$, then for all $ t \geq 0$, we have 
    \begin{align*}
        \Pr[X \geq t] \leq \exp(-t^2 / 2\sigma^2 )  .
    \end{align*}
\end{lemma}

\begin{lemma}[folklore]
\label{lem:anti_gaussian}
Let $X \sim {\cal N}(0,\sigma^2)$,
that is,
the probability density function of $X$ is given by $\phi(x)=\frac 1 {\sqrt{2\pi\sigma^2}}e^{-\frac {x^2} {2\sigma^2} }$.
Then
\begin{align*}
    \Pr[|X|\leq t] \leq \frac{4}{5} \frac{t}{\sigma}. 
\end{align*}
\end{lemma}

We state a tool for Chi-square distribution:
\begin{lemma}[Lemma 1 on page 1325 of Laurent and Massart \cite{lm00}]\label{lem:chi_square_tail}
Let $X \sim {\cal X}_k^2$ be a chi-squared distributed random variable with $k$ degrees of freedom. Each one has zero mean and $\sigma^2$ variance. Then
\begin{align*}
\Pr[ X - k \sigma^2 \geq ( 2 \sqrt{kt} + 2t ) \sigma^2 ] \leq \exp (-t), \\
\Pr[ k \sigma^2 - X \geq 2 \sqrt{k t} \sigma^2 ] \leq \exp(-t).
\end{align*}
\end{lemma}

 Matrix concentration inequalities have a large number of applications, for more details, we refer the readers to a survey by Tropp \cite{t15}. Recently, there are several non-trivial generalizations, e.g., Expander walk \cite{glss18,nrr19}, Strongly Rayleigh distributions \cite{ks18}, and matrix Poincare inequality \cite{aby19}. Here, we state matrix Bernstein inequality, which can be thought of as a matrix generalization of Lemma~\ref{lem:bernstein}.
\begin{lemma}[Matrix Bernstein, Theorem 6.1.1 in \cite{t15}]\label{lem:matrix_bernstein}
Consider a set of $m$ i.i.d. matrices $\{ X_1, \cdots, X_m \} \subset \R^{n_1 \times n_2}$. Assume that
\begin{align*}
\E[ X_i ] = 0, \forall i \in [m] ~~~ \mathrm{and}~~~ \| X_i \| \leq M, \forall i \in [m] .
\end{align*}
Let $X = \sum_{i=1}^m X_i$. Let $\mathrm{Var} [ X ] $ be the matrix variance statistic of sum:
\begin{align*}
\mathrm{Var} [X] = \max \left\{ \Big\| \sum_{i=1}^m \E[ X_i X_i^\top ] \Big\| , \Big\| \sum_{i=1}^m \E [ X_i^\top X_i ] \Big\| \right\}.
\end{align*}
Then 
\begin{align*}
\E[ \| X \| ] \leq ( 2 \mathrm{Var} [X] \cdot \log (n_1 + n_2) )^{1/2} +  M \cdot \log (n_1 + n_2) / 3.
\end{align*}
Furthermore, for all $t \geq 0$,
\begin{align*}
\Pr[ \| X \| \geq t ] \leq (n_1 + n_2) \cdot \exp \left( - \frac{t^2/2}{ \mathrm{Var} [ X ] + M t /3 }  \right)  .
\end{align*}
\end{lemma}

\newpage

\section{Application of concentration inequality}\label{sec:app_conc}
\subsection{Application of concentration inequality, truncated Gaussian}

\begin{lemma}[Inner product between two vectors]\label{lem:concentration_of_inner_product}
Let $a > 0$. Let $u_1, \cdots, u_d$ denote i.i.d. random variables satisfying $\forall i \in [d]$ $u_i = y_i \cdot z_i$ where $y_i \sim {\cal N}(0,\sigma^2)$ and 
\begin{align*}
z_i = 
\begin{cases}
1, & |y_i| \leq a ; \\
0, & |y_i| > a .
\end{cases}
\end{align*}
Then, for any fixed vector $x \in \R^d$, for any failure probability $\delta \in (0,1/10)$, we have
\begin{align*}
\Pr_{u}[ | \langle u, x \rangle | \geq 10 \| x\|_2 a (\sqrt{a/\sigma}+1) \log(1/\delta) ] \leq \delta.
\end{align*}
\end{lemma}
\begin{proof}

First, we can compute can $\E[ u_i ]$
\begin{align*}
\E[ u_i ] = \E[u_i] = 0 .
\end{align*}
Second, we can upper bound $\E[ (u_i)^2 ]$ using Lemma~\ref{lem:anti_gaussian}
\begin{align*}
\E[ (u_i)^2 ] = & ~ \E[ u_i^2 ] \\
\leq & ~ a^2 \cdot \Pr[ |u_i| \leq a ] \\
\leq & ~ a^2 \cdot \frac{4}{5} \frac{a}{\sigma}\\
\leq & ~ a^3 / \sigma .
\end{align*}
Third, we can upper bound $|u_i x_i|$ by $a \cdot \|x\|_{\infty}$.

Using Bernstein inequality, we have
\begin{align*}
\Pr[ | \langle u , x \rangle | \geq t ] 
\leq & ~ \exp \Big( -\frac{t^2/2}{ \| x\|_2^2 \E[u_i^2] + a \|x\|_{\infty} t/3 } \Big) \\
\leq & ~ \exp \Big( -\frac{t^2/2}{ \| x\|_2^2 a^3/\sigma + a \|x\|_{\infty} t /3 } \Big) .
\end{align*}
Choosing 
\begin{align*}
t= 5 \| x\|_2 a^{1.5} \sigma^{-0.5} \sqrt{ \log (1/\delta)} + 5 \| x\|_{\infty} a \log (1/\delta) 
\end{align*}
gives us
\begin{align*}
\Pr[ |\langle u, x \rangle|\geq 10 \| x\|_2 a (\sqrt{a/\sigma}+1) \log(1/\delta) ] \leq \delta .
\end{align*}

\end{proof}
\begin{lemma}[Matrix vector multiplication]\label{lem:matrix_vector_multiplication}
Let $a > 0$. Let $A_{i,j}$ denote i.i.d. random variables satisfying $\forall i \in [m], j \in [d]. $ $A_{i,j} = y_{i,j} \cdot z_{i,j}$ where $y_{i,j} \sim {\cal N}(0,\sigma^2)$ and 
\begin{align*}
z_{i,j} = 
\begin{cases}
1, & |y_{i,j}| \leq a ; \\
0, & |y_{i,j}| > a .
\end{cases}
\end{align*}
Then, for any fixed vector $x \in \R^d$, for any failure probability $\delta \in (0,1/10)$, we have
\begin{align*}
\Pr_{A} \Big[ | \| A x \|_2^2 - \E[ \|A x \|_2^2] | \geq 1000 m \| x\|_2^2 (\sigma^2 + a^2 ) \log^3(m/\delta) \Big] \leq \delta.
\end{align*}
Further, if $m = \Omega( \epsilon^{-2} \| x \|_2^2 (1 + a^2 / \sigma^2 )\log^3(m / \delta) )$, 
\begin{align*}
\Pr_{A} \Big[ \frac{1}{m} \big| \| A x \|_2^2 - \E[ \|A x \|_2^2] \big| \geq \epsilon^2 \| x\|_2^2 \sigma^2 \Big] \leq \delta.
\end{align*}

\end{lemma}
\begin{proof}

We define random variable $b_i = (Ax)_i^2$. 
We can upper bound $\E[ b_i ]$
\begin{align*}
\E[ b_i ] = \E [ (Ax)_i^2 ] \leq \| x \|_2^2 \cdot a^3 / \sigma .
\end{align*}
Similarly,
\begin{align*}
\E[ b_i ] = \E [ (Ax)_i^2 ] \geq 0.1 \| x\|_2^2 \cdot a^3 / \sigma .
\end{align*}

Next, we want to upper bound $\E[ b_i^2 ]$, for simplicity, let $u$ denote the $i$-th row of matrix $A$,
\begin{align*}
\E[ b_i^2 ] - (\E[ b_i ])^2
 = & ~  \E [ \langle u , x \rangle^4 ] - ( \E [ \langle u , x \rangle^2 ] )^2 \\
 = & ~ \E \Big[ ( \sum_{i=1}^d u_i x_i )^4 \Big] - \Big( \E \Big[ ( \sum_{i=1}^d u_i x_i )^2 \Big] \Big)^2 .
\end{align*}
For the first term, we have
\begin{align*}
\E \Big[ ( \sum_{i=1}^d u_i x_i )^4 \Big] 
= & ~ \E \Big[ \sum_{i=1}^d u_i^4 x_i^4 \Big] + 3\E \Big[ \sum_{i=1}^d \sum_{j \in [d] \backslash \{ i \} } u_i^2 x_i^2 u_j^2 x_j^2 \Big] \\
\leq & ~ \E[u_i^4] \cdot \| x \|_4^4 + 3 ( \E[ u_i^2 ] )^2 \cdot \| x \|_2^4 \\
\leq & ~ 4 \E[u_i^4] \cdot \| x \|_2^4.
\end{align*}
For the second term, we have
\begin{align*}
\Big( \E \Big[ ( \sum_{i=1}^d u_i x_i )^2 \Big] \Big)^2 = \Big( \sum_{i=1}^d \E[u_i^2]x_i^2 \Big)^2 = ( \E[u_i^2] )^2 \cdot \| x \|_2^4 .
\end{align*}
Thus, we have
\begin{align*}
\E[ b_i^2 ] - (\E[ b_i ])^2 \leq 4 \E[u_i^4] \cdot \| x \|_2^4 \leq 64 \sigma^4 \| x \|_2^4. 
\end{align*}

We also need to upper bound $|b_i|$. Apply Lemma~\ref{lem:concentration_of_inner_product}, we have, for a fixed $i \in [m]$,
\begin{align*}
|b_i| \leq (10 \| x \|_2 a (\sqrt{a/\sigma}+1) \log(m/\delta))^2 := b_{\max}
\end{align*} 
holds with probability at least $1 - \delta / m$.

Taking a union bound over $m$ coordinates, with probability $1-\delta$, we have : for all $i \in [m]$, $|b_i| \leq b_{\max}$.

Applying Bernstein inequality (Lemma~\ref{lem:bernstein}) on $\sum_{i=1}^m b_i$ again
\begin{align*}
\Pr \Big[ \Big| \sum_{i=1}^m (b_i -\E [ b_i ] ) \Big| \geq t \Big]
 \leq & ~ \exp \Big( - \frac{t^2/2}{ \sum_{i=1}^m \mathrm{Var}[  b_i  ] + b_{\max} t /3} \Big) \\
 \leq & ~ \exp \Big( - \frac{t^2/2}{ 64 m \sigma^4 \| x \|_2^4 + b_{\max} t /3} \Big) .
\end{align*}
Choosing $t= 50 m \sigma^2 \| x \|_2^2 \log (1/\delta) + 50 m b_{\max} \log(1/\delta) $, we complete the proof.
\end{proof}

\subsection{Application of concentration inequalities, classical random Gaussian}

\begin{lemma}[Inner product between a random guassian vector and a fixed vector]\label{lem:concentration_of_inner_product_guassian}
    Let $a > 0$. Let $u_1, \cdots, u_d$ denote i.i.d. random guassian variables where $u_i \sim {\cal N}(0,\sigma_1^2)$.
        
    Then, for any fixed vector $e \in \R^d$, for any failure probability $\delta \in (0,1/10)$, we have
        \begin{align*}
        \Pr_{u} \Big[ | \langle u, e \rangle | \geq  2 \sigma_1 \|e\|_2 \sqrt{ \log(d/\delta)} + \sigma_1\|e\|_{\infty}\log^{1.5}(d/\delta) \Big] \leq \delta.
        \end{align*}
        \end{lemma}
        \begin{proof}
        
        First, we can compute $\E[ u_i ]$
        \begin{align*}
        \E[ u_i ] = \E[u_i] = 0 .
        \end{align*}
        Second, we can compute $\E[ (u_i)^2 ]$
        \begin{align*}
        \E[ (u_i)^2 ] = & ~ \E[ u_i^2 ] = \sigma_1^2 .
        \end{align*}
        Third, we can upper bound $|u_i|$ and $|u_i e_i|$.
        \begin{align*}
            \Pr_{u}[|u_i-\E[u_i]| \geq t_1] &\leq \exp\Big(-\frac{t_1^2}{2\sigma_1^2}\Big) .
        \end{align*}
        Take $t_1=\sqrt{2\log(d/\delta)}\sigma_1$, then for each fixed $i\in [d]$, we have, $|u_i| \leq \sqrt{2\log(d/\delta)}\sigma_1$ holds with probability $1-\delta/d$.

        Taking a union bound over $d$ coordinates, with probability $1-\delta$, we have : for all $i \in [d]$, $|u_i| \leq \sqrt{2\log(d/\delta)}\sigma_1$.

        Let $E_1$ denote the event that, $\max_{i\in[d]}|u_ie_i|$ is upper bounded by $\sqrt{2\log(d/\delta)}\sigma_1\|e\|_{\infty}$. $\Pr[E_1] \geq 1-\delta$.
    
        Using Bernstein inequality, we have
        \begin{align*}
        \Pr_{u}[ | \langle u , e \rangle | \geq t ] 
        \leq & ~ \exp \Big( -\frac{t^2/2}{ \| e\|_2^2 \E[u_i^2] +  \max_{i\in[d]} |u_i e_i| \cdot t/3 } \Big) \\
        \leq & ~ \exp \Big( -\frac{t^2/2}{ \| e\|_2^2 \sigma_1^2 + \sqrt{2\log(d/\delta)}\sigma_1\|e\|_{\infty}\cdot t/3 } \Big) \\
        \leq & ~ \delta,
        \end{align*}
        where the second step follows from $\Pr[E_1] \geq 1-\delta$ and $\E[u_i^2] = \sigma_1^2$, and the last step follows from choice of $t$:
        \begin{align*}
        t= 2\sigma_1 \|e\|_2 \sqrt{ \log(d / \delta) } + \sigma_1\|e\|_{\infty}\log^{1.5}(d/\delta) .
        \end{align*}
       
       Taking a union with event $E_1$, we have
        \begin{align*}
        \Pr[ |\langle u, e \rangle| \geq t ] \leq 2\delta  .
        \end{align*}
        Rescaling $\delta$ completes the proof.

    \end{proof}

\begin{lemma}[Inner product between two random guassian vectors]\label{lem:concentration_of_inner_product_two_guassian}
Let $a > 0$. Let $u_1, \cdots, u_d$ denote i.i.d. random Gaussian variables where $u_i \sim {\cal N}(0,\sigma_1^2)$ and $e_1, \cdots, e_d$ denote i.i.d. random Gaussian variables where $e_i \sim {\cal N}(0,\sigma_2^2)$
    
Then, for any failure probability $\delta \in (0,1/10)$, we have
    \begin{align*}
    \Pr_{u,e} \Big[ | \langle u, e \rangle | \geq 10^4 \sigma_1 \sigma_2 \sqrt{d} \log^2(d/\delta) \Big] \leq \delta.
    \end{align*}
    \end{lemma}
    \begin{proof}
    First, using Lemma~\ref{lem:chi_square_tail}, we compute the upper bound for $\|e\|_2^2$
    \begin{align*}
        \Pr_{e}[\|e\|_2^2 - d\sigma_2^2 \geq (2\sqrt{dt}+2t)\sigma_2^2] \leq \exp(-t).
    \end{align*}
    
    Take $t = \log(1/\delta)$, then with probability $1-\delta$, 
    \begin{align*}
    \|e\|_2^2 \leq (d+3\sqrt{d\log(1/\delta)}+2\log(1/\delta))\sigma_2^2 \leq 4d \log(1/\delta) \sigma_2^2.
    \end{align*}
    Thus
    \begin{align*}
    \Pr_{e} [ \|e\|_2 \leq 4 \sqrt{d \log(1/\delta)} \sigma_2 ] \geq 1- \delta.
    \end{align*}
     
    Second, we compute the upper bound for $\|e\|_{\infty}$ (the proof is similar to Lemma~\ref{lem:concentration_of_inner_product_guassian})
    \begin{align*}
        \Pr_{e}[|\|e\|_{\infty} \leq \sqrt{\log(d/\delta)} \sigma_2] &\geq 1- \delta.
    \end{align*}
    
    We define $t$ and $t'$ as follows
    \begin{align*}
    t= & ~ 4 \cdot ( \sigma_1 \|e\|_2 \sqrt{ \log(d/\delta)} + \sigma_1\|e\|_{\infty}\log^{1.5}(d/\delta) ) \\
    t' = & ~ 8 \cdot ( \sigma_1\sigma_2\sqrt{d}\log(d/\delta) + \sigma_1\sigma_2 \log^2(d/\delta) ) .
    \end{align*}
    From the above calculations, we can show
    \begin{align*}
    \Pr_{e}[t' \geq t] \geq 1 - 2\delta.
    \end{align*}

    By Lemma \ref{lem:concentration_of_inner_product_guassian}, for fixed $e$, we have
    \begin{align*}
         \Pr_{u}[ | \langle u , e \rangle | \geq t ] \leq \delta .
    \end{align*}
          
    Overall, we have
    \begin{align*}
         \Pr_{e,u}[ | \langle u , e \rangle | \geq t' ] \leq 3 \delta .
    \end{align*}
    Rescaling $\delta$ completes the proof.
\end{proof}

\begin{lemma}[Concentration of folded Gaussian]\label{lem:anti_concentration_ax}
    Let matrix $A \in \R^{m \times d}$ be defined as each entry is i.i.d. random variables satisfying $\forall i \in [m]$, $j \in [d]$. $A_{i,j} = y_{i,j}$ where $y_{i,j} \sim {\cal N}(0,\sigma_A^2)$. Let $\ov{A} \in \R^{m \times d}$ be defined as, $\forall i \in [m], j\in [d]$, $\overline{A}_{i,j} = y_{i,j}\cdot z_{i,j}$ where
    \begin{align*}
    z_{i,j} = 
    \begin{cases}
    1, & \mathrm{~if~} 0 \leq y_{i,j} \leq a; \\
    0, & \mathrm{~otherwise~}.
    \end{cases}
    \end{align*}
    Let $x \in \R^d_{+}$ denote a non-negative vector where $\sum_{i=1}^d x_i = 1$. 
    
    1) For any failure possibility $\delta \in (0, 1/10)$, we have
    \begin{align*}
        \Pr \left[ \forall i \in [m], (\bar{A}x)_i \geq \sigma_A \cdot C \right] >  1- \delta,
    \end{align*}
    where
    \begin{align*}
    C: = \frac{ a^2 }{ 6\sigma_A^2 } - ( \frac{ 2 a^3 }{ 9 \sigma_A^3 } )^{1/2} \cdot \sqrt{ \log(m/\delta) } - \frac{ 2a }{ 3\sigma_A } \cdot \log(m/\delta) .
    \end{align*}
    2) For any failure possibility $\delta \in (0, 1/10)$, if $a/\sigma_A \geq 20 \log(m/\delta)$, then
    \begin{align*}
        \Pr \left[ \forall i \in [m], (\bar{A}x)_i \geq \sigma_A \cdot 0.02 \cdot (a^2/\sigma_A^2) \right] >  1- \delta.
    \end{align*}
\end{lemma}
\begin{proof}
For a fixed $i \in [m]$, for each $j \in [d]$, we define
\begin{align*}
    b_j = \bar{A}_{i,j} x_j.
\end{align*}

We first calculate $\E[b_j]$, $\E[b_j^2]$ and $\var[b_j]$.

We provide a lower bound for $\E[ b_j ]$,
\begin{align*}
    \E[b_j] = & ~ \E[A_{i,j}]x_j \\
            = & ~ x_j \int_{0}^{a} \frac{1}{\sigma_A\sqrt{2\pi}}\exp(-x/\sigma_A^2) x \d x \\ 
            \geq & ~ \frac{a^2x_j}{2\sigma_A\sqrt{2\pi}}  \\
            \geq & ~ \frac{a^2x_j}{6\sigma_A} .
\end{align*}

We give an upper bound for $\E[b_j^2]$,
\begin{align*}
    \E[b_j^2] 
    = & ~ \E[A_{i,j}^2]x_j^2 \\
    = & ~ x_j^2 \int_{0}^{a} \frac{1}{\sigma_A\sqrt{2\pi}}\exp(-x/\sigma_A^2) x^2 \d x \\\leq & ~ \frac{a^3x_j^2}{3\sigma_A\sqrt{2\pi}}  \\
    \leq & ~ \frac{a^3x_j^2}{9\sigma_A} .
\end{align*}

We can upper bound $\var[b_j]$,
\begin{align*}
    \var[b_j] =& \E[b_j^2] - \E[b_j]^2 \leq  \E[b_j^2] \leq \frac{a^3x_j^2}{9\sigma_A}.
\end{align*}
Then, we can lower bound $\sum_{j=1}^d\E[b_j]$
\begin{align*}
    \sum_{j=1}^d \E[b_j] 
    \geq \frac{a^2}{6\sigma_A} \sum_{j=1}^d x_j 
    = \frac{a^2}{6\sigma_A},
\end{align*}
where the last step follows from $\sum_{j=1}^d x_j = 1$.

Next, we can upper bound $b_j$ and $\sum_{j=1}^d\var[b_j]$
\begin{align*}
     M: = \max_{j \in [d]} b_j \leq & ~ \max_{j \in [d]} x_j a \leq a.
\end{align*}

\begin{align*}
    \sum_{j=1}^d\var[b_j] \leq & \sum_{j=1}^d \frac{a^3x_j^2}{9\sigma_A} \leq \sum_{j=1}^d \frac{a^3x_j}{9\sigma_A} = \frac{a^3}{9\sigma_A}.
\end{align*}

Applying Bernstein inequality (Lemma~\ref{lem:bernstein}) on $\sum_{j=1}^d ( b_j -\E[b_j] ) $
\begin{align*}
    \Pr \Big[ \sum_{j=1}^d ( b_j - \E[b_j] ) \leq -t \Big] 
    \leq & \exp \Big( - \frac{t^2/2}{\sum_{j=1}^d \var[b_j] + M t/3} \Big) \\
    \leq & \exp \Big( - \frac{t^2/2}{a^3/9\sigma_A + a t/3} \Big) .
\end{align*}

Taking 
\begin{align*}
t = \sigma_A \cdot( \sqrt{2a^3/9\sigma_A^3 \cdot \log(m/\delta)}+ 2a/3\sigma_A \cdot \log(m/\delta)),
\end{align*}
then for any $i \in [m]$,
\begin{align*}
    \Pr \Big[ \sum_{j=1}^d b_j \geq a^2 / ( 6 \sigma_A ) - t \Big] 
    \geq & ~ \Pr \Big[ \sum_{j=1}^d b_j \geq \sum_{j=1}^d \E[b_j]-t \Big] \\ 
    \geq & ~ 1- \delta ,
\end{align*}
where the first step holds because $\sum_{j=1}^d \E[b_j] > a^2/(6\sigma_A)$. 

Since $( \bar{A} x )_i = \sum_{j=1}^d b_j$, we have for any $i \in [m]$,
\begin{align*}
    \Pr \left[ ( \bar{A} x )_i \geq \sigma_A \cdot \Big( a^2 / 6 \sigma_A^2 - \sqrt{ ( 2 a^3 / 9 \sigma_A^3 ) \cdot \log ( m / \delta ) } - ( 2 a / 3 \sigma_A ) \cdot \log( m / \delta ) \Big) \right] >  1- \delta / m .
\end{align*}
Taking a union bound over all $i \in [m]$ completes the proof.
\end{proof}

\newpage
\section{Anti-concentration}\label{sec:app_anti}

Given a number of independent random variables, 
the well-known Central Limit Theorem (CLT) states that their sum has good concentration under certain conditions.
Such concentration results like the Chernoff bound \cite{c52} and Hoeffding's inequality \cite{h63}
are among the central tools in Theoretical Computer Science (TCS). From the opposite perspective, we can also ask for \emph{anti-concentration} results.
For example,
let $x$ be a Rademacher variable (choosing $\pm 1$ with probability $1/2$) and let $a$ denote a vector in $\R^d$. The celebrated Littlewood-Offord Lemma states that any $d$-variate degree-$1$ polynomial $p(x) = \sum_{i=1}^d a_i x_i$ does not concentrate on any particular value.
\begin{theorem}[Littlewood and Offord \cite{lo43}]\label{thm:lo43}
Let $C > 0$ denote a universal constant. For any linear form $p$ satisfying $|a_i| \geq 1$, $\forall i \in [d]$, and any open interval $I$ of length $1$, we have
\begin{align*}
    \Pr_{ x \sim \{ -1, +1 \}^d } [ p(x) \in I ] \leq C \cdot \frac{ \log d }{ \sqrt{d} } .
\end{align*}
\end{theorem}
Two years later, Erd\"{o}s \cite{e45} removed the $\log d$ factor in Theorem~\ref{thm:lo43}. Recently,  Theorem~\ref{thm:lo43} has been generalized to higher degree polynomials by \cite{ctv06,rv13,mnv17}.

Instead of considering $x_i$ as $\{-1,+1\}$ random variables, Carbery and Wright \cite{cw01} showed the anti-concentration result for $x_i$ chosen as i.i.d. Gaussians.
\begin{theorem}[Carbery and Wright \cite{cw01}]\label{thm:cw_uf}
Let $p : \R^d \rightarrow \R$ denote a degree-$k$ polynomial with $d$ variables. There is a universal constant $C > 0$ such that
\begin{align*}
\Pr_{x \sim {\cal N}(0,I_d) } \Big[ | p( x ) | \leq \delta \sqrt{ \Var[ p( x ) ] } \Big] \leq C \cdot \delta^{1/k}.
\end{align*}
\end{theorem}
These are worst-case results in the sense that they hold for arbitrary polynomials.
For example,
Theorem \ref{thm:cw_uf} is tight for any polynomial that is a perfect $k$-th power.

We can generalize Theorem~\ref{thm:cw_uf} into the following\footnote{The generalization also has been observed in \cite{syz20}, for the completeness, we provide the proof here.}:
\begin{lemma}[An variation of \cite{cw01}, Anti-concentration of sum of truncated Gaussians]
Let $x_1, \cdots, x_n$ be $n$ i.i.d. zero-mean Gaussian random variables ${\cal N}(0,1)$. Let $p : \R^n \rightarrow \R$ denote a degree-1 polynomial defined as
\begin{align*}
    p(x_1,\cdots,x_n)=\sum_{i=1}^n \alpha_i x_i.
\end{align*} Let $f$ denote a truncation function where $f(x) = x$ if $|x| \leq a$, and $f(x) = 0$ if $|x| > a$. Then we have
\begin{align*}
\Pr_{ x \sim {\cal N}(0,I_d) } \Big[ | p( f( x ) ) | \leq \min\{a,0.1\} \cdot \delta \cdot \| \alpha \|_2 \Big] \geq C \cdot \delta.
\end{align*}
\end{lemma}

\begin{proof}
Let $\mu:\mathbb{R}^n\rightarrow\mathbb{R}_{\geq 0}$ be the truncated Gaussian distribution.
We first argue that $\mu$ is log-concave.
Indeed,
for any $x,y\in \mathbb{R}^n$ and $\lambda\in [0,1]$,
if $\mu(x)=0$ or $\mu(y)=0$,
then we must have 
\begin{align*}\mu(\lambda x+(1-\lambda)y)\geq 0=(\mu(x))^{\lambda} \cdot (\mu(y))^{1-\lambda}.
\end{align*}
On the other hand,
if $\mu(x)>0$ and $\mu(y)>0$,
then we must have $\mu(\lambda x+(1-\lambda)y)>0$,
because 
\begin{align*}
\|\lambda x+(1-\lambda)y\|_2\leq \lambda \|x\|_2+(1-\lambda)\|y\|_2,
\end{align*}
hence $\mu$ would not truncate at $\lambda x+(1-\lambda)y$.
Notice that Gaussian distribution is log-concave.
Let $\mu':\mathbb{R}^n\rightarrow \mathbb{R}$ be the density function of Gaussian distribution,
then $\mu(x)= C_0 \cdot \mu'(x)$ for some universal constant $C_0 >0$ for all $x$ that is not truncated.
so in this case we still have
\begin{align*}
    \mu(\lambda x+(1-\lambda)y)
    = & ~C_0 \cdot \mu'(\lambda x+(1-\lambda)y)\\
    \geq & ~C_0 \cdot (\mu'(x))^{\lambda} \cdot (\mu'(y))^{1-\lambda}\\
    = & ~(C_0 \mu'(x))^{\lambda} \cdot ( C_0 \mu'(y))^{1-\lambda}\\
    = & ~(\mu(x))^{\lambda} \cdot (\mu(y))^{1-\lambda}.
\end{align*}

So we conclude that $\mu$ is log-concave.

Now we apply Theorem~\ref{thm:cw_better} on $\mu$ and $p$.
By setting $q=2$ and $d=1$,
 we have
 \begin{align}\label{eq:cw_app}
     \Big( \int_{x\in \mathbb{R}^n} |p(x)|^{2} \d \mu \Big)^{1/2} \cdot \mu(|p(x)|\leq \alpha)\leq C \cdot \alpha.
 \end{align}
 Notice that
 \begin{align*}
     \int_{x\in \mathbb{R}^n} |p(x)|^{2} \d \mu
     = & ~\E_{x \sim \mu} \Big[ \Big( \sum_{i=1}^n \alpha_ix_i \Big)^2 \Big]\\
     = & ~\sum_{i=1}^n\alpha_i^2\E_{ x  \sim \mu }[ x_i^2 ]\\
     = & ~\sum_{i=1}^n \alpha_i^2\Var_{ x_i \sim \mu_i }[ x_i ] ,
 \end{align*}
 where $\mu_i: \mathbb{R} \rightarrow \mathbb{R}$ is the distribution on the $i$-th coordinate, $\forall i \in [n]$. 
Hence we can rewrite Eq. \eqref{eq:cw_app} as
\begin{align*}
    \Pr_{ x \sim {\cal N}(0,I_d) } \left[ \Big| \sum_{i=1}^n \alpha_if(x_i) \Big| \leq \delta \left(\sum_{i=1}^n \alpha_i^2\Var_{ x_i \sim \mu_i }[ x_i ] \right)^{1/2}\right] \geq C \cdot \delta.
\end{align*}

By Claim~\ref{clm:truncated_var}, we have
\begin{align*}
    \Pr_{x\sim \N(0,I_d)}\left[ |p(f(x))|\leq \delta \left(\sum_{i=1}^n \alpha_i^2\cdot \left(1-\sqrt{\frac{2}{\pi}} \cdot \frac{a\cdot  e^{-a^2/2}}{\mathrm{erf}(a/\sqrt{2})}\right)\right)^{1/2} \right]\geq C\cdot \delta.
\end{align*}

For $0\leq a \ll 1$, 
\begin{align*}
    1-\sqrt{\frac{2}{\pi}} \cdot \frac{a\cdot  e^{-a^2/2}}{\mathrm{erf}(a/\sqrt{2})} = \frac{5}{6}a^2+o(a^3).
\end{align*}
Hence, 
\begin{align*}
    \Pr_{x\sim \N(0,I_d)} \Big[  |p(f(x))|\leq \delta a \| \alpha \|_2 \Big] \geq C \cdot \delta.
\end{align*}

For $a\geq 1$,
\begin{align*}
    1-\sqrt{\frac{2}{\pi}} \cdot \frac{a\cdot  e^{-a^2/2}}{\mathrm{erf}(a/\sqrt{2})} = \Theta\left(1- a e^{-a^2} -  e^{-a^2/2} / a \right).
\end{align*}
Hence,
\begin{align*}
    \Pr_{x\sim \N(0,I_d)}\left[ |p(f(x))|\leq \delta ( 1 - a e^{-a^2} -  e^{-a^2/2} / a )^{1/2} \| \alpha \|_2 \right]\geq C\cdot \delta.
\end{align*}

When $a\geq 1$,
we have $0.025\leq (1-a e^{-a^2} -  e^{-a^2/2} / a )\leq 1$.
So we can combine the above two cases to get
\begin{align*}
    \Pr_{x\sim \N(0,I_d)}\Big[ |p(f(x))| \leq  \min\{a,0.1\} \cdot \delta \| \alpha \|_2  \Big]\geq C\cdot \delta.
\end{align*}
\end{proof}

\begin{claim}\label{clm:truncated_var}
Let $x\in \R$ be a standard Gaussian random variable $\N(0,1)$. Let $f$ denote a truncation function where $f(x) = x$ if $|x| \leq a$, and $f(x) = 0$ if $|x| > a$. Then, we have
\begin{align*}
    \Var[f(x)] =~ 1-\sqrt{\frac{2}{\pi}} \cdot \frac{a\cdot  e^{-a^2/2}}{\mathrm{erf}(a/\sqrt{2})},
\end{align*}
where $\mathrm{erf}(x) = \frac{2}{\sqrt{\pi}}\int_0^x e^{-t^2} \d t$.
\end{claim}

\begin{theorem}[\cite{ak18}]\label{thm:cw_better}
Let $\mu:\mathbb{R}^n\rightarrow \mathbb{R}$ be a log-concave measure over $\mathbb{R}^n$.
Let $L^{1}(\mu)=\int_{x\in \mathbb{R}^n} |\mu(x)|  \d x$.
For any $q>0$ and polynomial $p:\mathbb{R}^n\rightarrow \mathbb{R}$,
define the $\ell_q$ norm of $p$ with respect to the measure $\mu$ as
\begin{align*}
    \|p\|_q=\left(\int p^q\d \mu\right)^{1/q}.
\end{align*}
Assume $p$ has degree $d$.
Then there exists constant $C(d)>0$ that only depends on $d$ so that for all $\alpha>0$ and all $q>0$,
\begin{align*}
    \Big( \int |p(x)|^{q/d} \d \mu \Big)^{1/q} \cdot \mu(|p(x)|\leq \alpha)\leq C(d) \cdot \alpha^{1/d}.
\end{align*}
\end{theorem}
\newpage
\section{Sensitivity}\label{sec:app_sens}

\subsection{Concentration of folded Gaussian}

\begin{lemma}[concentration of folded gaussian]\label{lem:concentration_of_folded_gaussian}
    Let matrix $A \in \R^{m \times d}$ be defined as each entry is i.i.d. random variables satisfying $\forall i \in [m]$, $j \in [d]$. $A_{i,j} = y_{i,j}$ where $y_{i,j} \sim {\cal N}(0,\sigma_A^2)$. Let $z_{i,j} = |y_{i,j}|$,
    then $\forall j \in [d]$, 
    \begin{align*}
     \Pr \Big[\sum_{i=1}^m z_{i,j} - \E[z_{i,j}] > \sigma_A m+4\sigma_A\sqrt{m}\log^{1.5}(md/\delta) \Big] \leq \delta .
    \end{align*}
\end{lemma} 
\begin{proof}
    For a fixed $j$, let $b_i = z_{i,j}$. First we calculate $\E[b_i]$
    \begin{align*}
        \E[b_i] 
        = & ~ \int_{0}^{\infty} \frac{2}{\sqrt{2\pi\sigma_A^2}} \exp(-x^2/2\sigma_A^2) x \d x \\
        = & ~ \sigma_A \sqrt{2/\pi} .
    \end{align*}
    
    Second, we calculate $\E[b_i^2]$
    \begin{align*}
        \E[b_i^2] = \E[z_{i,j}^2] = \E[y_{i,j}^2] = \sigma_A^2 .
    \end{align*}

    According to Lemma \ref{lem:concen_gaussian}, we can upper bound $z_{i,j}$
    \begin{align*}
        \Pr[z_{i,j} > t] = \Pr[|y_{i,j}| > t] \leq \exp(-t/\delta_2^2) .
    \end{align*}
    Taking $t=\sigma_A\sqrt{\log(md/\delta_1)}:=M$, we have $\forall i \in [m]$, $j \in [d]$
    \begin{align*}
        \Pr[\max_{i,j}{z_{i,j}} > t] \leq \delta_2 .
    \end{align*}
    
    Applying Bernstein inequality on $\sum_{i=1}^m b_i$
    \begin{align*}
        \Pr \Big[ \Big| \sum_{i=1}^m (b_i -\E [ b_i ] ) \Big| \geq t \Big]
        \leq & ~ \exp \Big( - \frac{t^2/2}{ \sum_{i=1}^m \mathrm{Var}[  b_i  ] + b_{\max} t /3} \Big) \\
        \leq & ~ \exp \Big( - \frac{t^2/2}{ m\sigma_A^2 +  \sigma_At\sqrt{\log(md/\delta_2)} /3} \Big) .
    \end{align*}
    Choosing $t = \sigma_A m+4\sigma_A\sqrt{m}\log^{1.5}(md/\delta)$, we have 
    \begin{align*}
        \Pr \Big[\sum_{i=1}^m z_{i,j} - \E[z_{i,j}] > t \Big] \leq \delta .
    \end{align*}
\end{proof}

\subsection{$\ell_1$-sensitivity functions of single layer neural network}
\begin{lemma}[$\ell_1$-norm sensitivity of single layer neural network]\label{lem:gs1_single_layer}
    Let $x \in [0,1]^d$, fully connected matrix $A \in {\cal N}(0,\sigma_A)^{m \times d}$, bias matrix $b \in \mathbb{R}^{m}$, and $\phi$ is the ReLU activation function. Let $f(x) = \phi(Ax+b)$ denote a single layer network, then for all neighboring inputs $x_1, x_2 \in \R^d$ that differ at most in one entry, we have
    \begin{align*}
     \Pr \Big[ \GS_1(f) \leq \sigma_A m+4\sigma_A\sqrt{m}\log^{1.5}(md/\delta) \Big] \geq 1-\delta.
     \end{align*} 
\end{lemma}
\begin{proof}
    Let $k$ denote the index that $x_1$ and $x_2$ are different.
    \begin{align*}
            \GS_1(f) 
            = & ~ \sup_{x_1,x_2\in \R^d} \|f(x_1)-f(x_2) \|_1 \\
            = & ~ \sup_{x_1,x_2\in \R^d} \|\phi(Ax_1+b) - \phi(Ax_2+b)\|_1 \\
            \leq & ~ \sup_{x_1,x_2\in \R^d} \|(Ax_1+b) - (Ax_2+b)\|_1 \\
            = & ~ \sup_{x_1,x_2\in \R^d} \|(A(x_1- x_2)\|_1 \\
            = & ~ \| A_{*,k} \|_1 \\
            \leq & ~ \sigma_A m+4\sigma_A\sqrt{m}\log^{1.5}(md/\delta),
    \end{align*} 
    where the fourth step follows that $x_1$ and $x_2$ differ in the $k$-th entry, and the fifth step follows Lemma~\ref{lem:concentration_of_folded_gaussian}.
\end{proof}

\subsection{$\ell_2$-sensitivity functions of single layer neural network}
\begin{lemma}[$\ell_2$-norm sensitivity of single layer neural network]\label{lem:gs2_single_layer}
    Let $x \in [0,1]^d$, fully connected matrix $A \in {\cal N}(0,\sigma_A)^{m \times d}$, bias matrix $b \in \mathbb{R}^{m}$, and $\phi$ is the ReLU activation function. Let $f(x) = \phi(Ax+b)$ denote a single layer network, then for all neighboring inputs $x_1, x_2 \in \R^d$ that differ at most in one entry, we have
    \begin{align*}
        \Pr \Big[ \GS_2(f) \leq 2( \sqrt{md}+\sqrt{\log(1/\delta)} ) \Big] \geq 1-\delta.
     \end{align*} 
\end{lemma}
\begin{proof}
    Let $k$ denote the index that $x_1$ and $x_2$ are different.
    \begin{align*}
            \GS_2(f) 
            = & ~ \sup_{x_1,x_2\in \R^d} \|f(x_1)-f(x_2) \|_2 \\
            = & ~ \sup_{x_1,x_2\in \R^d} \|\phi(Ax_1+b) - \phi(Ax_2+b)\|_2 \\
            \leq & ~ \sup_{x_1,x_2\in \R^d} \|(Ax_1+b) - (Ax_2+b)\|_2 \\
            = & ~ \sup_{x_1,x_2\in \R^d} \|(A(x_1- x_2)\|_2 \\
            = & ~ \| A_{*,k} \|_2  \\
            \leq & ~ \sigma_A\left( 2\sqrt{md\log(1/\delta)}+2\log(1/\delta)+md\right)^{1/2} \\
            \leq & ~ \sigma_A\left( 2\sqrt{2md\log(1/\delta)}+2\log(1/\delta)+md\right)^{1/2} \\
            = & ~ \sigma_A(\sqrt{md}+\sqrt{2\log1/\delta}),
    \end{align*} 
    where the fourth step follows that $x_1$ and $x_2$ differ in the $k$-th entry, and the fifth step follows Lemma~\ref{lem:chi_square_tail}.
\end{proof}
\newpage
\section{Equivalence between pruning and differential privacy}
\label{sec:app_dp}

\subsection{Main results}

\begin{table}\caption{Summary of two results}
\centering
\begin{tabular}{ | l | l | l | l | }
    \hline
  {\bf Statement} & $\epsilon_{\dfp}$ & {\bf Comment} & {\bf Pruning} \\\hline
  Theorem~\ref{thm:for_general_x} & $\GS_1(f)  / (\sigma \sigma_A) \cdot (m/\delta_{\dfp}) $ & General $x$ & Magnitude \\\hline
  Theorem~\ref{thm:for_positive_x} & $\GS_1(f)  / (\sigma \sigma_A) \cdot \log (m/\delta_{\dfp}) $ & Nonnegative $x$ & Folded Magnitude \\\hline
\end{tabular}

\end{table}

\begin{theorem}[Main result I]\label{thm:for_general_x}
    For a single layer neural network $f(x) = \phi ( A x + b )$ where fully connected matrix $A \in {\cal N}(0,\sigma_A^2)^{m \times d}$, vector $b \in \mathbb{R}^{ m }$, and $\phi$ is the ReLU activation function.  We assume all the inputs $x \in \R^d$ satisfying that $\|x\|_2 = 1$. 
    If
    \begin{align*}
    m = \Omega ( \poly(\epsilon_{\ap}^{-1}, \log(1/\delta_{\ap}), \log(1/\delta_{\mathrm{dp}}) , a/ \sigma_A, \sigma \sigma_A)),
    \end{align*}
    then applying magnitude pruning with with truncation threshold $a > 0$ on $A \in \R^{m \times d}$ is an $(\epsilon_{\ap},\delta_{\ap})$-approximation to applying $(\epsilon_{\mathrm{dp}},\delta_{\mathrm{dp}})$-differential privacy on $x$,  where \\$\epsilon_{\mathrm{dp}}=2 \GS_1(f)  (m/\delta_{\mathrm{dp}}) / ( \sigma \sigma_A )$.
\end{theorem}

\begin{theorem}[Main result II]\label{thm:for_positive_x}
    For a single layer network $f(x)=\phi(Ax+b)$ where fully connected matrix $A \in \mathcal{N}(0,\sigma_A^2)^{m \times d}$, vector $b \in \mathbb{R}^{m }$, and $\phi$ is the ReLU activation function. We assume all the inputs $x \in \R^d$ satisfying that $\|x\|_2 = 1$ and $x \in \R^d_{+}$. 
    If 
    \begin{align*}
    m = \Omega ( \poly(\epsilon_{\ap}^{-1}, \log(1/\delta_{\ap}), \log(1/\delta_{\mathrm{dp}}) , a/ \sigma_A, \sigma \sigma_A)),
    \end{align*}
    then applying folded magnitude pruning with truncation threshold $a > 0$ on $A \in \R^{m \times d}$ is an $(\epsilon_{\ap},\delta_{\ap})$-approximation to applying $(\epsilon_{\mathrm{dp}},\delta_{\mathrm{dp}})$-differential privacy on $x \in \R^d$, where \\$\epsilon_{\mathrm{dp}}=2 \GS_1(f) \log(m/\delta_{\mathrm{dp}}) / ( \sigma \sigma_A )$.
\end{theorem}

\begin{remark}
Note that $\GS_1(f) = \Theta(m \sigma_A)$. \\
1) if using folded Gaussian and assume $x \in \R_{\geq 0}$, $\epsilon_{\mathrm{dp}}=2 \GS_1(f) \cdot \log(m/\delta_{\mathrm{dp}}) / ( \sigma \sigma_A )$,\\
then we need to pick $\sigma=m$, $\sigma_A = \Theta(1/\sigma)$ and $a = \Theta(\sigma_A)$ . \\
2) if using Gaussian, $\epsilon_{\mathrm{dp}}=2 \GS_1(f) \cdot (m/\delta_{\mathrm{dp}}) / ( \sigma \sigma_A )$,\\ 
then we need to pick $\sigma=m^2$, $\sigma_A = \Theta(1/\sigma)$ and $a = \Theta(\sigma_A)$.
\end{remark}

\subsection{Differential privacy}
\label{def:differential_privacy}
\begin{definition}[Differential Privacy, Definition.1 in \cite{dmns06}]
Let $\mathcal{A}: \mathcal{D}^{n} \rightarrow \mathcal{Y}$ be a randomized algorithm. Let $D_{1}, D_{2} \in \mathcal{D}^{n}$ be two databases that differ in at
most one entry (we call these databases neighbors). Let $\epsilon>0 .$ Define $\mathcal{A}$ to be $\epsilon$ -differentially private if for all neighboring databases
\(D_{1}, D_{2},\) and for all (measurable) subsets \(Y \subset \mathcal{Y},\) we have
\begin{align*}
\frac{ \Pr \left[\mathcal{A}( D_{1} ) \in Y \right] }{\Pr \left[\mathcal{A}\left(D_{2}\right) \in Y\right]} \leq \exp (\epsilon) .
\end{align*}
\end{definition}

\begin{definition}[Global Sensitivity, Definition 2 in \cite{dmns06}] 
Let $f : {\cal D}^n \rightarrow \mathbb{R}^d$, define $\GS_p(f)$, the $\ell_p$ global sensitivity of $f$, for all neighboring databases $D_1, D_2$ as
\begin{align*}
     \GS_p(f) = \sup_{ D_1,D_2 \in {\cal D}^n } \| f(D_1) - f(D_2) \|_p .
\end{align*}
\end{definition}

\begin{theorem}[Laplace Mechanism \cite{dmns06}]
    Let $f$ be defined as before and $\epsilon>0$. Define randomized algorithm $\mathcal{A}$ as
    \begin{align*}
        \mathcal{A}(D) = f(D) + ( \mathrm{Lap}( \GS_1(f) / \epsilon ) )^d ,
    \end{align*}
    where the one-dimensional (zero mean) Laplace distribution $\mathrm{Lap}(b)$ has density $p(x;b)=\frac{1}{2b} \exp(-\frac{|x|}{b})$, and $\mathrm{Lap}(b)^d=(l_1,\dots,l_d) \in \R^d$ where each $l_i$ i.i.d. is sampled from $\mathrm{Lap}(b)$. Then $\mathcal{A}$ is $\epsilon$-differentially private.
\end{theorem}

\begin{theorem}[Gaussian Mechanism \cite{dr14}]
    For $c>2\sqrt{\log(1/\delta)}$, the Gaussian Mechanism with parameter $\sigma \geq c\cdot \GS_2(f)/\epsilon$ is $(\epsilon,\delta)$-differentially private.
    
\end{theorem}

\subsection{Function approximation}

\begin{definition}[$(\epsilon,\delta)$-approximation] For a pair of functions $f(x)$ and $g(x)$, we say $f$ is an $(\epsilon,\delta)$-approximation of $g$ if for any $x$
    \begin{align*}
        \Pr[ \| f(x) - g(x) \|_2 > \epsilon ] \leq \delta .
    \end{align*}
    
\end{definition}

\subsection{Proof of Theorem~\ref{thm:for_positive_x}}

\begin{proof}

{\bf Sketch.}

The proof can be splitted into two parts. We use $\tilde{A} \in \R^{m \times d}$ to denote the weight matrix after magnitude pruning, and $\bar{A} = \tilde{A} - A \in \R^{m \times d} $. We define vector $e \in \R^m$ as follows
\begin{align*}
e = \mathrm{Lap}(1,\sigma)^m\circ ( \bar{A}x ).
\end{align*}
    \begin{enumerate}
        \item Let $B(x)=f(x)+e \in \R^m$, then $B(x)$ is $( \epsilon_{\mathrm{dp}}, \delta_{\mathrm{dp}})$-differential privacy.
        \item $\Pr[ \frac{1}{\sqrt{m}} \| e-\bar{A}x \|_2 \geq \epsilon_{\ap} ] \leq \delta_{\ap}$, as long as $m = \Omega ( \poly(\epsilon_{\ap}^{-1}, \log(1/\delta_{\ap}), a/\sigma_A, \sigma \sigma_A))$.
    \end{enumerate}

{\bf Part 1.}

Let $y \in \mathbb{R}^m$ and $x_1, x_2$ be neighbouring inputs.  It is sufficient to bound the ratio $\frac{p(y-f(x_1))}{p(y-f(x_2))}$ where $p(\cdot)$ denotes probability density, because once the densities are bounded, integrating $p(\cdot)$ yields the requirement for differential privacy as defined in \ref{def:differential_privacy}.

Since $e_i \sim \mathrm{Lap}(1,\sigma)\cdot (\bar{A}x)_i$, then $p(t:\sigma)=\frac{1}{2\sigma} \exp(- | t / ( \bar{A}x )_i - 1 | / \sigma )$
    \begin{align*}
        \frac{p(y-f(x_1))}{p(y-f(x_2))} 
        = & ~ \frac{\prod_{i=1}^{m}\frac{1}{2\sigma}\exp(- |(y_i-f(x_1)_i)/(\bar{A}x_1)_i-1| / \sigma)}{\prod_{i=1}^{m}\frac{1}{2\sigma}\exp(- |(y_i-f(x_2)_i)/(\bar{A}x_2)_i-1|/\sigma )}\\
        = & ~ \frac{\exp(-\sum_{i=1}^m |(y_i-f(x_1)_i)/(\bar{A}x_1)_i-1|/\sigma)}{\exp(-\sum_{i=1}^m |(y_i-f(x_2)_i)/(\bar{A}x_2)_i-1|/\sigma)} \\
        = & ~ \exp \Big( \frac{1}{\sigma}\sum_{i=1}^m \Big| \frac{y_i-f(x_1)_i}{(\bar{A}x_1)_i}-1 \Big| - \Big| \frac{y_i-f(x_2)_i}{(\bar{A}x_2)_i}-1 \Big| \Big) \\
        \leq & ~ \exp \Big(\frac{1}{\sigma} \sum_{i=1}^m \Big| \frac{y_i-f(x_1)_i}{(\bar{A}x_1)_i}-\frac{y_i-f(x_2)_i}{(\bar{A}x_2)_i} \Big| \Big) \\
        \leq & ~ \exp \Big(2 \sum_{i=1}^m \frac{1}{\sigma \min_{i \in [m]} \{ |(\bar{A}x_1)_i|,|(\bar{A}x_2)_i| \}}|f(x_2)_i-f(x_1)_i| \Big) \\
        \leq & ~ \exp \Big( 2\GS_1(f) / \sigma\min_{i \in [m]} \{ |(\bar{A}x_1)_i|,|(\bar{A}x_2)_i| \} \Big) \\
        \leq & ~ \exp \Big( 2 \GS_1(f) /  \sigma_A \sigma (1/6\cdot (a/\sigma_A)^2 - 1/5 \cdot (a/\sigma_A)^{1.5}\log(m/\delta_{dp})) \Big) \\
        \leq & ~ \exp \Big( 2 (\sigma_A m+4\sigma_A\sqrt{m}\log^{1.5}(md/\delta)) /  \sigma_A \sigma (1/6\cdot (a/\sigma_A)^2 - 1/5 \cdot (a/\sigma_A)^{1.5}\log(m/\delta_{dp})) \Big) \\
        \leq & ~ \exp \Big( 2 (m+4\sqrt{m}\log^{1.5}(md/\delta)) /  \sigma (1/6\cdot (a/\sigma_A)^2 - 1/5 \cdot (a/\sigma_A)^{1.5}\log(m/\delta_{dp})) \Big)
    \end{align*} 
    where the first equality is because the noise is independent for each coordinate, and the first inequality is triangle inequality. The third inequality holds because of the definition of $\GS_1(f)$, and the fourth holds because of Lemma \ref{lem:anti_concentration_ax}.
holds with probability $1-\delta_{\mathrm{dp}}$

According to Lemma \ref{lem:gs1_single_layer}, 
\begin{align*}
\Pr \Big[ \GS_1(f) \leq \sigma_A m+4\sigma_A\sqrt{m}\log^{1.5}(md/\delta) \Big] \geq 1-\delta.
\end{align*}

{\bf Part 2.}
    
    Let $z_i = (e_i-\bar{A}x_i)^2$, thus $z_i \sim \mathrm{Lap}^2(0,b_i)$, where $b_i=\sigma(\bar{A}x)_i$
    We first calculate $\E[z_i^2]$
    \begin{align*}
        \E[z_i^2] = & ~ \int_{-\infty}^{\infty}\frac{1}{2b_i}\exp(-|x|/b_i)x^4 \d x \\
                = & ~ \int_{0}^{\infty}\frac{1}{b_i}\exp(-x/b_i)x^4 \d x \\
                = & ~ b_i^4 \cdot (-\exp(-x)x^4 - \int_{0}^{\infty} -4\exp(-x)x^3 \d x ) \bigg|_{0}^{\infty} \\
                = & ~ b_i^4 \cdot (-\exp(-x)x^4 +4(-\exp(-x)x^3-3(\exp(-x)x^2-2(-\exp(-x)x-\exp(-x)))) )|_{0}^{\infty} \\
                = & ~ 24b_i^4 \\
                \leq & ~ 24\sigma^4 \cdot ( 10a(\sqrt{a/\sigma_A}+1)\log(m/\delta) )^4  ,
    \end{align*}
    where both the third step and the fourth step follow integration by parts. The fifth step follows by plugging in the limits of integration, and the last step follows by Lemma \ref{lem:concentration_of_inner_product}.

    Next, we want to bound $\max(z_i)$, since $z_i = e_i^2 \sim \mathrm{Lap}^2(0,b_i)$
    \begin{align*}
        \Pr[z_i\geq t^2] = & \Pr[|e_i| \geq t]  &\text{$t>0$}\\
            = & 2\cdot\frac{1}{2}\exp(- t / b_i ) \\
            = & \exp(- t/b_i) ,
    \end{align*}
    where the second step follows by plugging the cumulative distribution function of Laplace distribution. 

    Take $t=b_i\log(m/\delta)$, then for each fixed $i \in [m]$, we have $\Pr[z_i\leq \sqrt{b_i\log(m/\delta ) }] = \delta/m$. Thus, with probability $1-\delta$, we have for all $i \in [m]$,
    \begin{align*}
    z_i \leq \max_{i\in [m]} \sqrt{ b_i \log(m/\delta ) } \leq \sqrt{ a^3\sigma/\sigma_A \cdot \log(m/\delta ) }:= M ,
    \end{align*}
    where the second inequality follows by $(\bar{A}x)_i$'s upper bound in Lemma \ref{lem:matrix_vector_multiplication}.

    Using Bernstein inequality, we have
    \begin{align*}
        \Pr \left[ \Big|\sum_{i=1}^{m}(z_i-\E[z_i] ) \Big| \geq t \right] \leq& \exp \Big( -\frac{t^2/2}{\sum_{i=1}^{m}\E[z_i^2]+Mt/3} \Big) \\
        \leq& \exp \Big(-\frac{t^2/2}{24m\sigma^4 \cdot ( 10a(\sqrt{a/\sigma_A}+1)\log(m/\delta) )^4 + a^3\sigma/\sigma_A \cdot \log(m/\delta )t/3} \Big) .
    \end{align*}
    Since $Mt/3$ is dominated by $\sum_{i=1}^{m}\E[z_i^2]$, we choose 
    \begin{align*}
    t = m \epsilon^2,
    \end{align*}
    then as long as 
    \begin{align*}
        m \geq \epsilon^{-4}\log(1/\delta) \Big(48 \cdot ( 10a/\sigma_A(\sqrt{a/\sigma_A}+1)\log(m/\delta) )^4 \Big) (\sigma^4 \sigma_A^4 + \sigma \sigma_A^2) ,
    \end{align*}
    we have
    \begin{align*}
     \Pr \left[ \frac{1}{m} \Big| \sum_{i=1}^{m}(z_i-\E[z_i] ) \Big| \geq \epsilon^2 \right] \leq \delta  .
    \end{align*}
    which is
    \begin{align*}
    \Pr \left[ \frac{1}{\sqrt{m}} \| e - \bar{A} x \|_2 \geq \epsilon \right] \leq \delta .
    \end{align*}
    Note that we need to pick $\sigma = m$, then we need to pick $\sigma_A = 1/m$.
\end{proof}
\newpage
\section{Experiment details}
\label{sec:app_exp}

\subsection{Network architecture and hyperparameters}
Table \ref{tab:exp_details} provides implementation details of the deep neural networks we use in experiments. Most of our experiments are conducted on 8 Nvidia Tesla K80 GPUs. Experiment scripts are written in Python 3.6.

\begin{table}[h]\caption{Implementation details of network architectures and training schemes.}
\label{tab:exp_details}
\centering
\begin{tabular}{ | l | l | l | l | }
    \hline
      & \textbf{MNIST} \cite{lcb10} & \textbf{CIFAR-10} \cite{cifar10} \\ \hline
    $\#$Epoch & 20 & 200\\ \hline
    Network architecture & LeNet-5 \cite{lbbh98} & VGG-19 \cite{sz14} \\\hline 
    Optimizer & SGD  (momentum = 0.9) \cite{qian1999momentum}   & Adam \cite{kingma2014adam}\\  \hline
    Initial learning rate & 0.1 & 0.001 \\ \hline
    Batch size & 64 & 64\\\hline
  \end{tabular}
\end{table}

\subsection{Pruning algorithm}
Algorithm~\ref{alg:mag_prune_full} provides the full version of Algorithm~\ref{alg:mag_prune}. The pruning procedure decides $a^{(t)}$, the magnitude threshold at time $t$, by the gradual pruning technique \cite{zg17}, which we have discussed in Section~\ref{sec:exp} (see  Figure~\ref{fig:grad_prune}).

\begin{algorithm}[h]
    \caption{ Stochastic Gradient Descent with Magnitude-based Pruning, detailed version of Algorithm~\ref{alg:mag_prune}}
    \label{alg:mag_prune_full}
    \begin{algorithmic}[1]{
    \Procedure{\textsc{SGDMagPrune}}{$ \{ x_i, y_i \}_{i \in [n]}, a, \eta$}
    \State \Comment{Loss function $\mathcal{L}: \R^{d_o} \times \R^{d_o} \rightarrow [0,1]$ } 
    \State Let $W^{(1)}$ denote a random initialization of neural network's weights, and $f (W, x)$ denotes the neural network.
    \State Let ${\cal D} = \{ (x_1, y_1), \cdots, (x_n,y_n) \} \subset \R^d \times \R^{d_o}$
    \For{$t = 1 \to T_{\text{train}}$} \Comment{Training stage} 
        \State Sample $(x,y)\sim {\cal D}$ uniformly at random
        \State $W^{(t+1)} \leftarrow W^{(t)} - \eta \cdot \frac{\partial \mathcal{L}( f(W,x) , y)}{ \partial W } |_{W = W^{(t)}} $
    \EndFor 
    \For{$t = T_{\text{train}} \to T_{\text{train}}+T_{\text{prune}}$}  \Comment{Pruning stage} 
        \State Sample a data $(x,y)$ from ${\cal D}$ uniformly at random
        \State $\wt{W}^{(t)} \leftarrow \textsc{ThresholdPrune}(W^{(t)},a^{(t)})$
        \State $W^{(t+1)} = \tilde{W}^{(t)} - \eta \cdot \frac{\partial \mathcal{L}( f(W,x) , y)}{ \partial W} |_{W = \wt{W}^{(t)}} $
    \EndFor 
    \State $T_{\text{end}} \gets T_{\text{train}}+T_{\text{prune}}$
    \State $\wt{W}^{(T_{\text{end}})} \leftarrow \textsc{ThPrune}(W^{(T_{\text{end}})},a^{ ( T_{\text{end}} ) })$
  
    \EndProcedure
    \Procedure{\textsc{ThPrune}}{$W,a$} 
        \For{ $l \in [L]$ }
        \For{$i,j$}
            \State $(\tilde{W}_l)_{i,j} \leftarrow 
            \begin{cases}
            (W_l)_{i,j}, & \text{~if~} |(W_l)_{i,j}| > a; \\
            0, & \text{~otherwise~}.
            \end{cases} 
            $ 
        \EndFor
        \EndFor
        \State \Return $\wt{W}$
    \EndProcedure}
    \end{algorithmic}
\end{algorithm}

\subsection{Visualization of inverted images}
\paragraph{Inversion from different layers.} Figure \ref{fig:mnist_single_image} and Figure \ref{fig:cifar_single_image} visualize the inverted images obtained by running the attack (see Section~\ref{sec:inversion_alg}) on $\Phi'_\text{prune}$ from different layers with different target sparsities of magnitude pruning. Each column implies the increasing difficulty of inverting deeper layers. Each row indicates that with a given layer, running inversion to generate $x$ becomes harder as the model is pruned with a higher target sparsity. As shown, for layer `Conv5' in LeNet-5, if we prune the model with sparsity $0.9$, then the attack fails by producing an almost all-black inversion. Similar phenomenon is observed when inverting layer `Conv5-1' in VGG-19 with sparsity $0.9$.

\begin{figure*}[!t]
    \centering
    \includegraphics[width=0.95\textwidth]{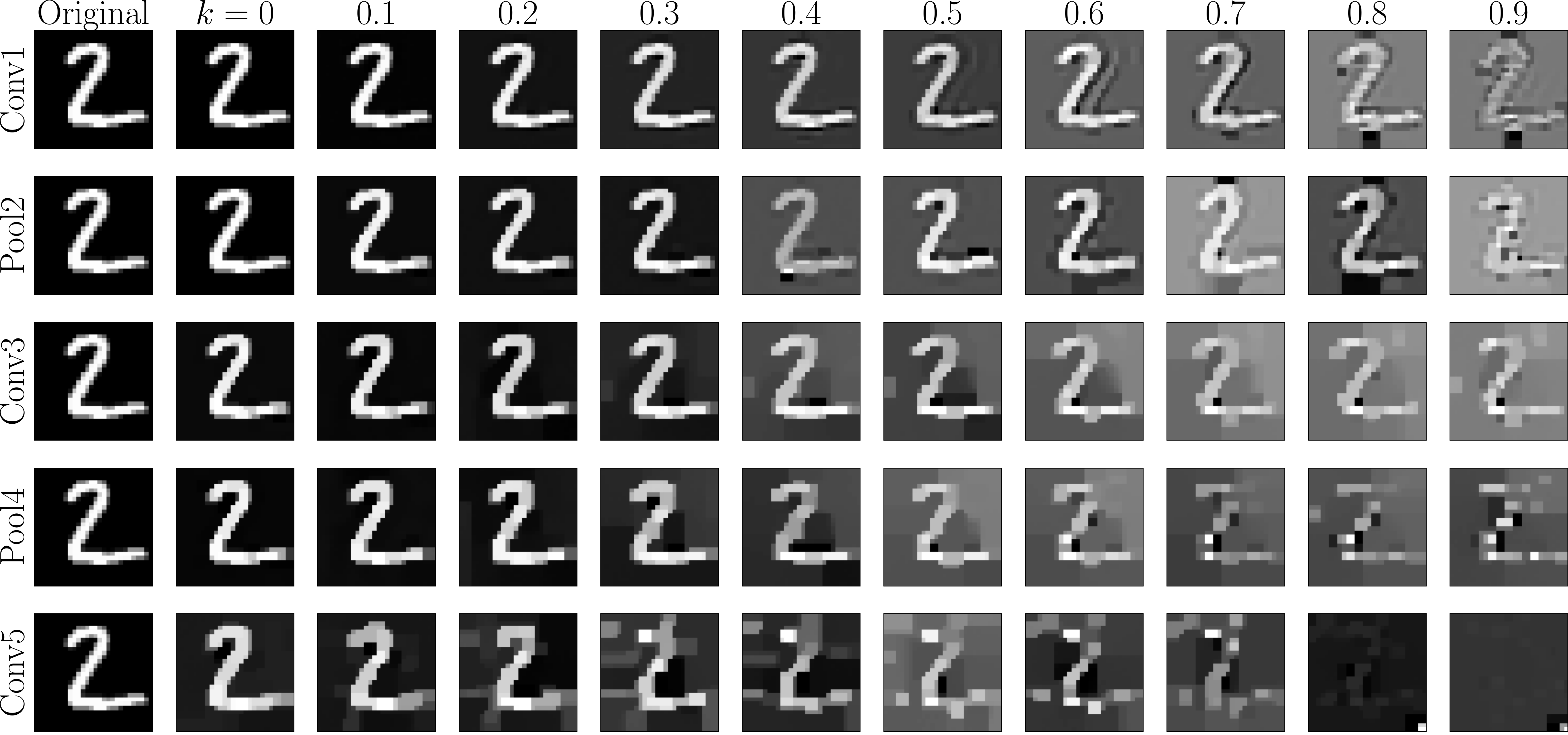}
    \caption{Inverted MNIST digits (a digit 2) from different layers with different sparsity levels. $k$ stands for sparsity.}\label{fig:mnist_single_image}
\end{figure*}

\begin{figure*}[!t]
    \centering
    \includegraphics[width=0.95\textwidth]{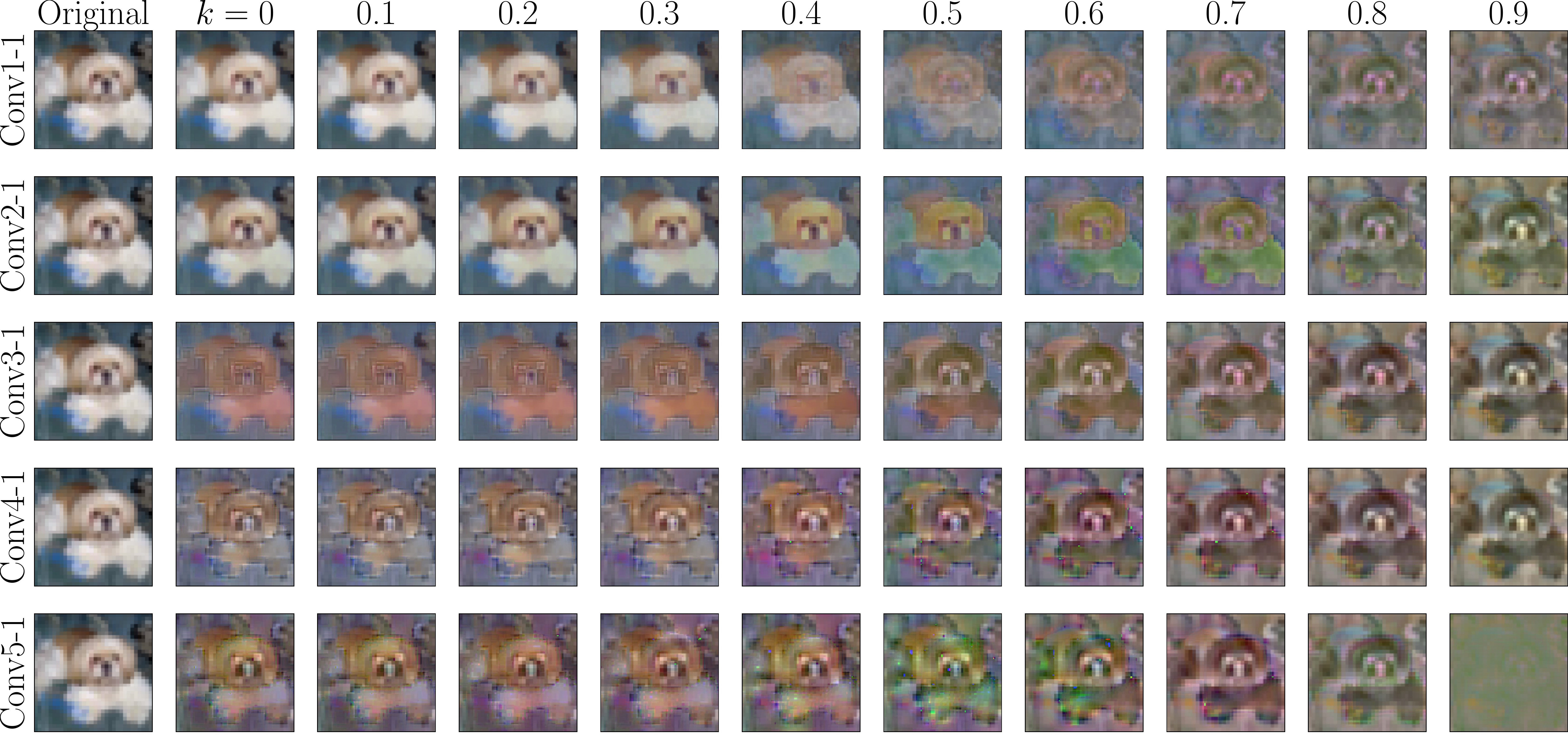}
    \caption{Inverted CIFAR-10 images (a dog) from different layers with different sparsity levels. $k$ stands for sparsity.}\label{fig:cifar_single_image}
\end{figure*}

\paragraph{Inversion from the same layer.} Figure~\ref{fig:multi_img_same_layer_mnist} and Figure~\ref{fig:multi_img_same_layer_cifar} show inverted images obtained by running multiple inversion attacks on the same layer (we show one sample from each class). Increasing the target sparsity of pruning makes the inversion attack harder for all classes. 

\begin{figure*}[!t]
\centering
\includegraphics[width=0.95\linewidth]{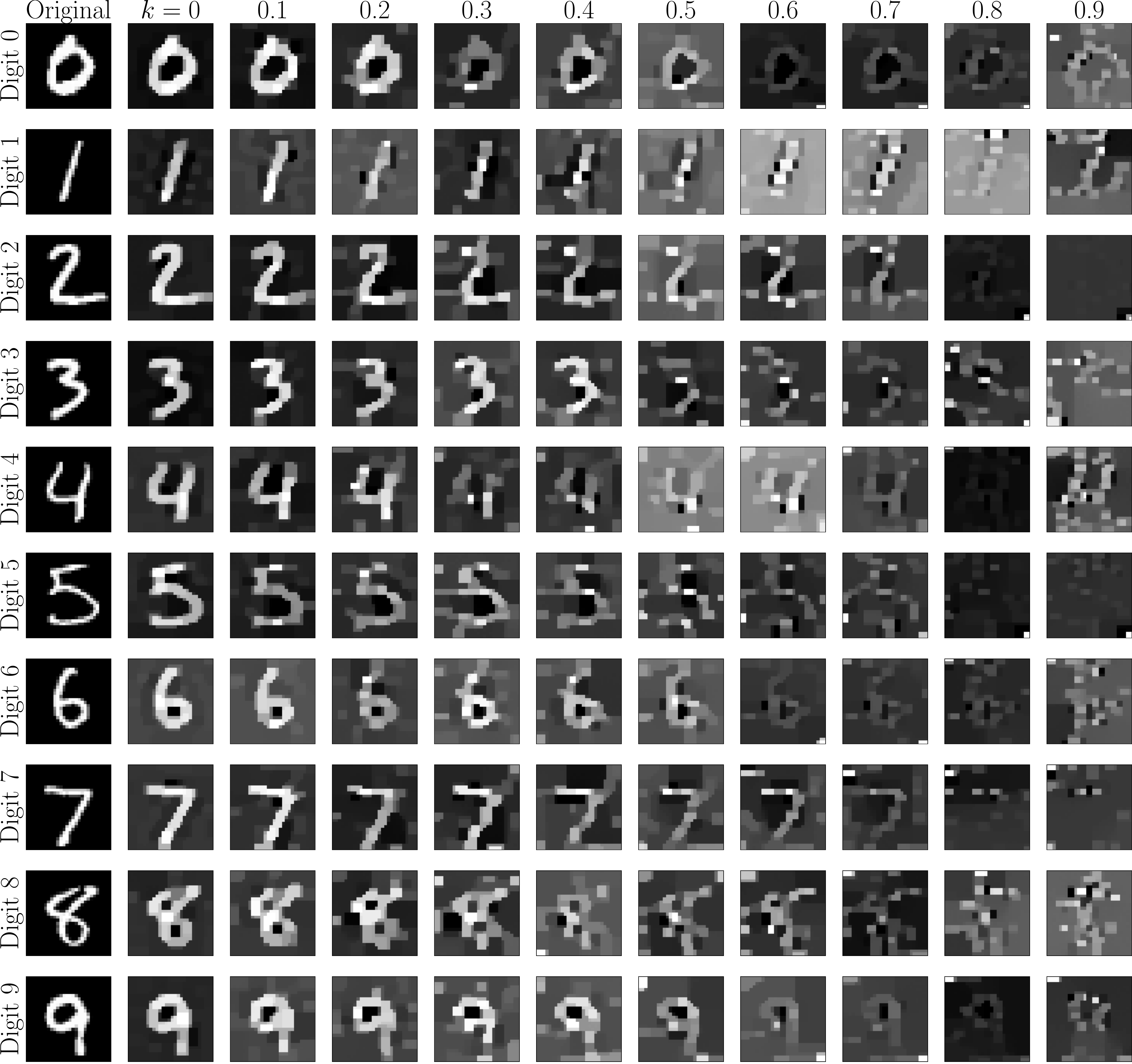}
\caption{Inversion from layer `Conv5' of LeNet-5 (MNIST). Each row shows a sample from each class.}
\label{fig:multi_img_same_layer_mnist}
\end{figure*}

\begin{figure*}[!t]
\centering
\includegraphics[width=0.95\linewidth]{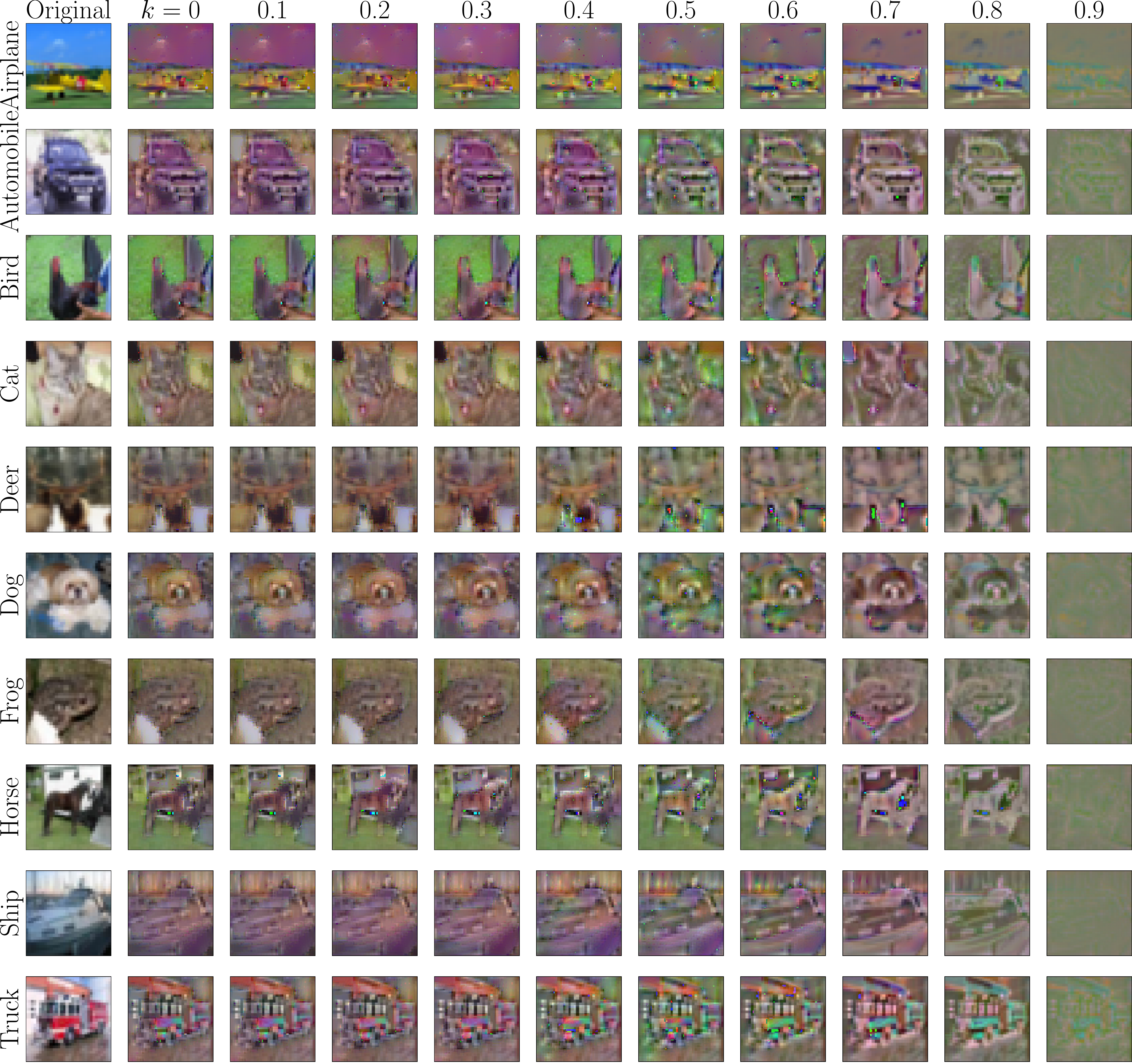}
\caption{Inversion from layer `Conv5-1' of VGG-19 (CIFAR-10). Each row shows a sample from each class.}
\label{fig:multi_img_same_layer_cifar}
\end{figure*}

\subsection{More on relative privacy leakage of pruning and adding noise}
Figure~\ref{fig:mnist_4metric} and Figure~\ref{fig:cifar_4metric} show the privacy leakage differences between magnitude-based pruning and adding differentially private noise for the experiments on MNIST and CIFAR-10 respectively (full version of Figure~\ref{fig:sim_mini_mnist} and Figure~\ref{fig:sim_mini_cifar}). As shown, given the same accuracy requirement, pruning helps preserve more privacy than adding noise.  When accuracy decreases from 0.9 to 0.1, the similarity curves of pruning and adding noise `converge'. This is because when the majority of the network parameters are 0's, or the hidden-layer output is dominated by the noise, running the inversion attack will hardly succeed in obtaining information about the original input.

Note that we did not provide the range for INFE similarity results because they all have minimum 0 and maximum 1 (even random guessing can make a correct prediction) and thus not informative.

\begin{figure*}[!t]
    \centering
        \subfloat{\includegraphics[width=0.46\linewidth]{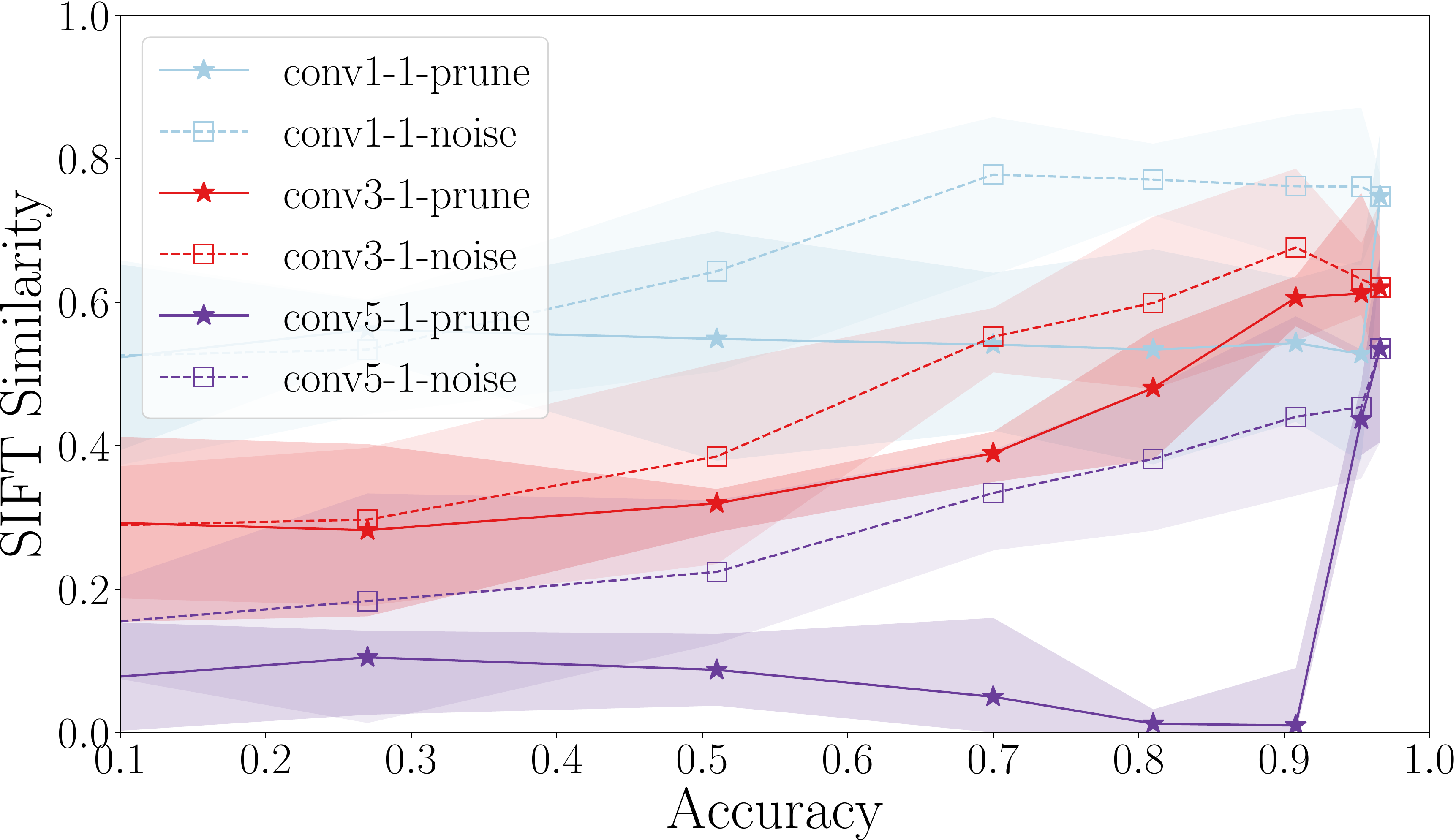}}
        \hspace{4mm}
        \subfloat{\includegraphics[width=0.46\linewidth]{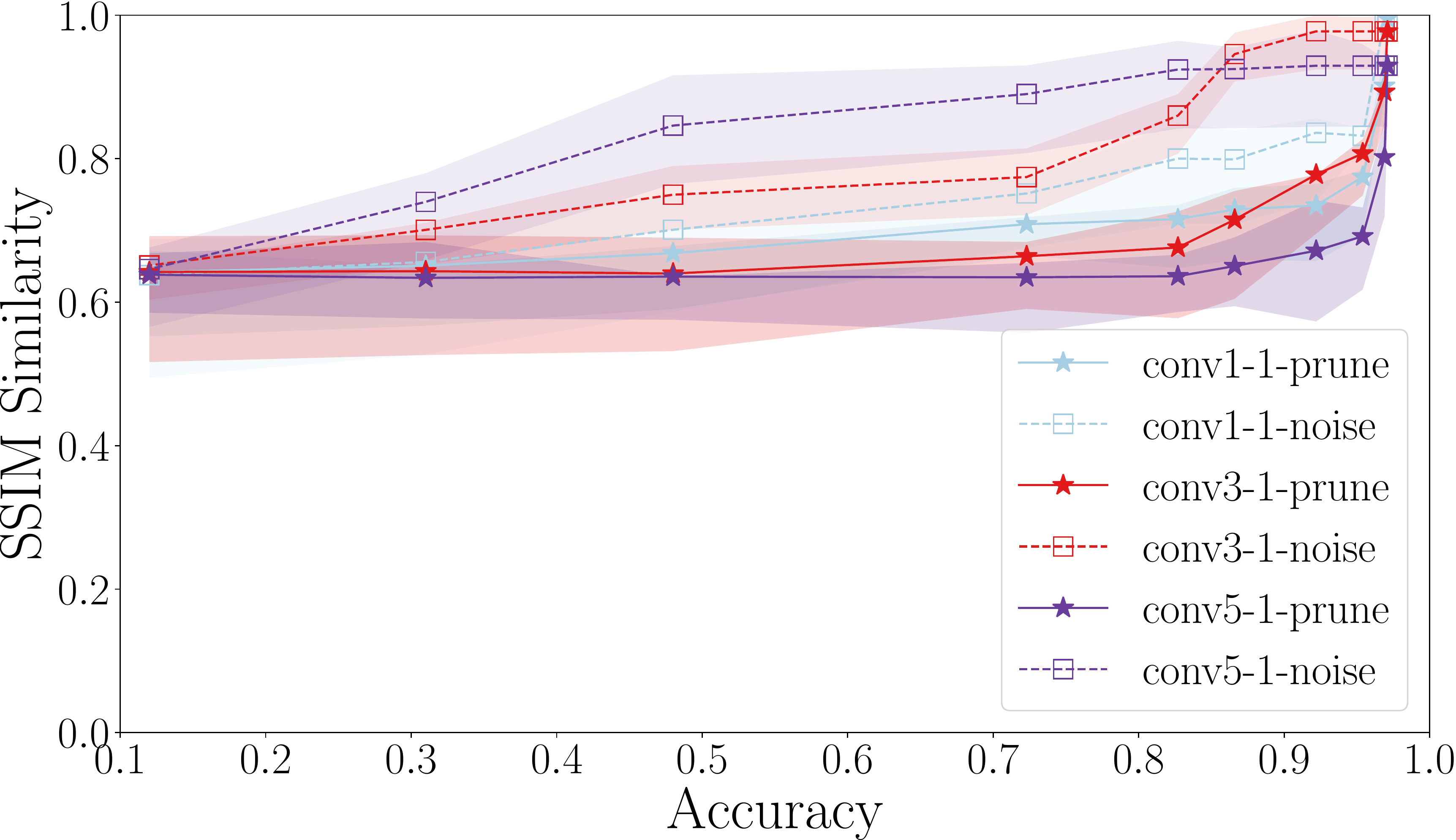}}
        \hspace{4mm}
        \subfloat{\includegraphics[width=0.46\linewidth]{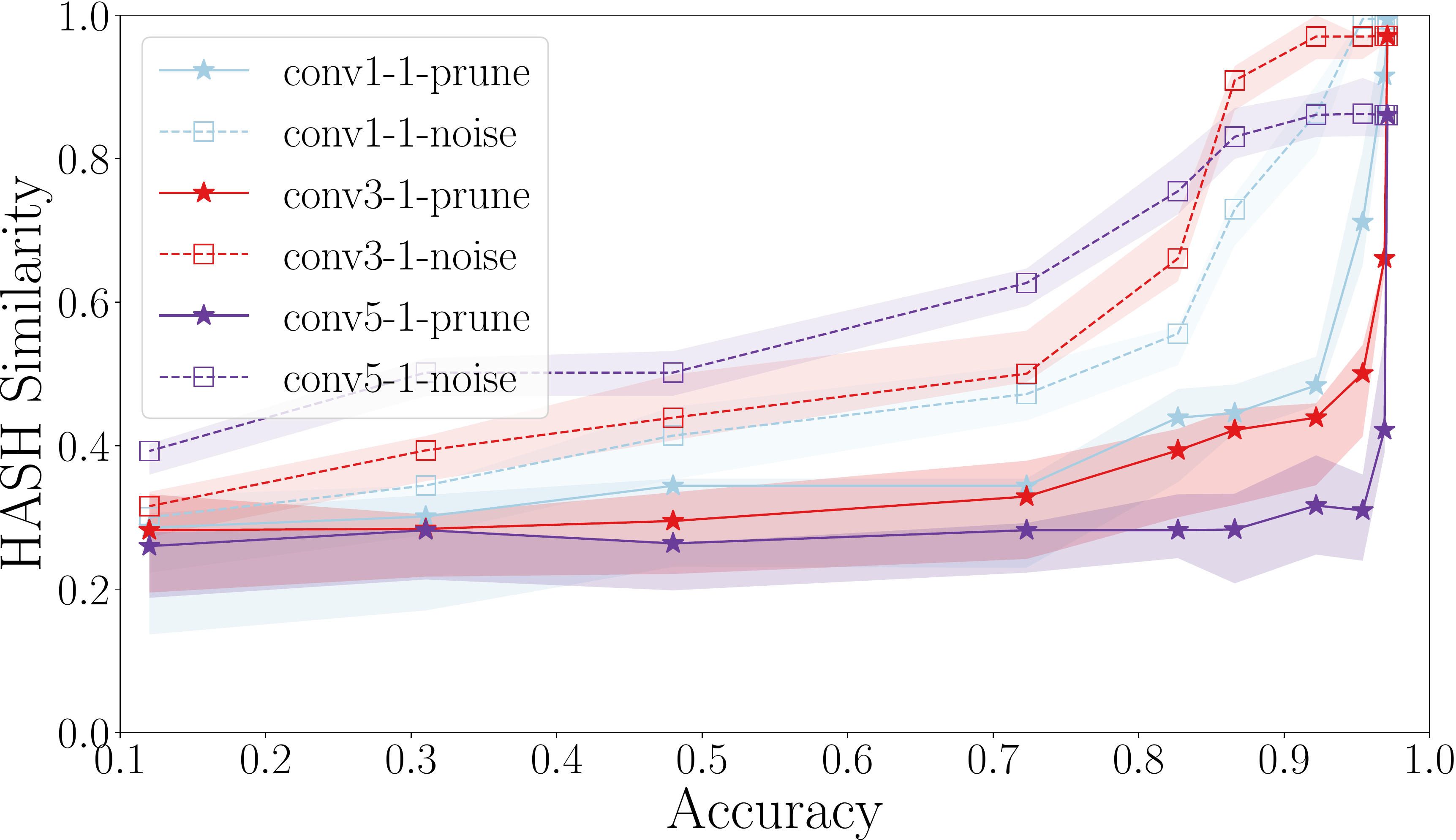}} 
        \hspace{4mm}
        \subfloat{\includegraphics[width=0.46\linewidth]{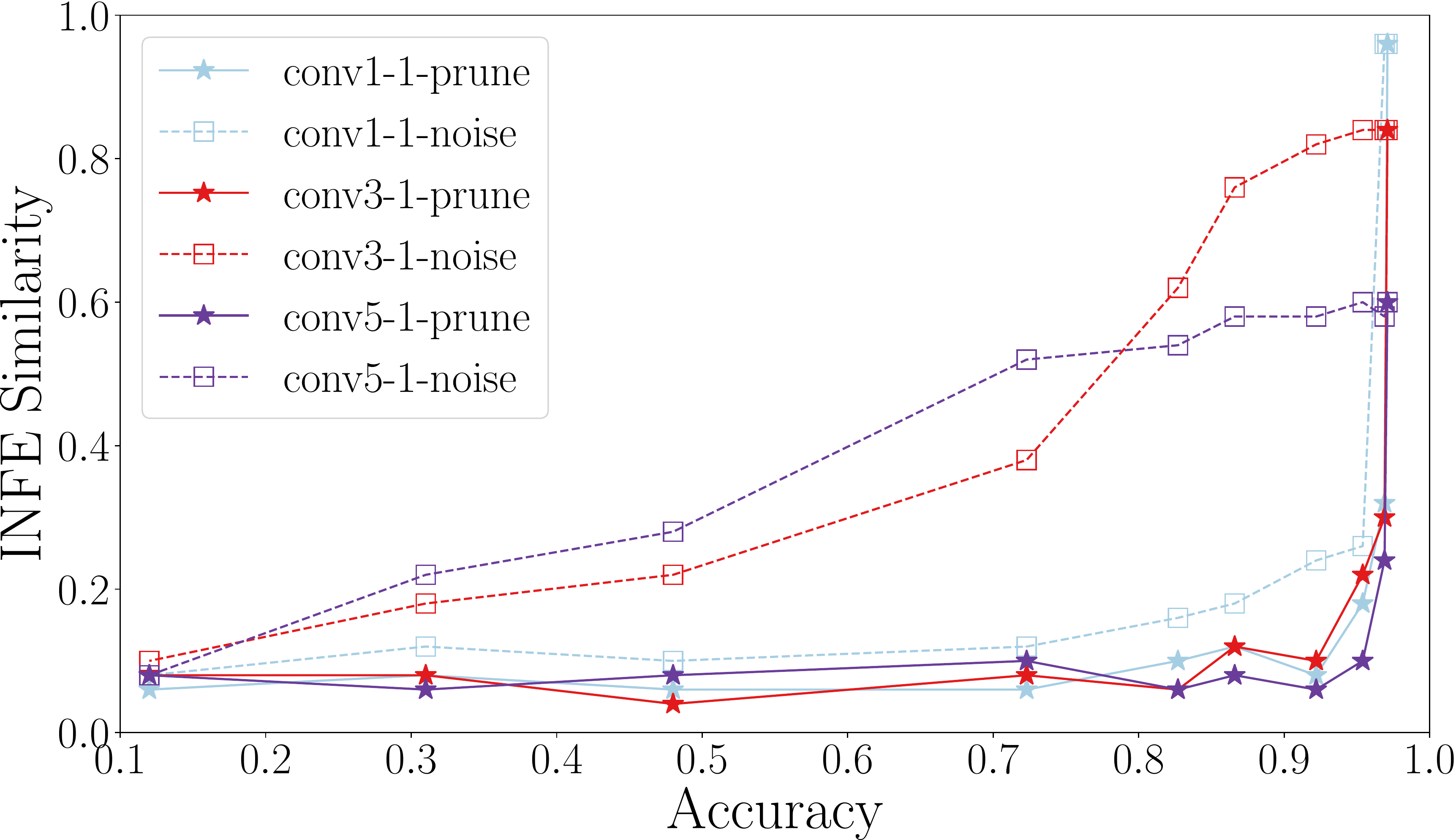}}
    \caption{Similarity between inversion and original images by applying pruning and adding noise on different layers (LeNet-5 + MNIST). Shadow represents value range.}
    \label{fig:mnist_4metric}
\end{figure*}

\begin{figure*}[!t]
    \centering
        \subfloat{\includegraphics[width=0.46\linewidth]{cifar_SIFT.pdf}}
        \hspace{4mm}
        \subfloat{\includegraphics[width=0.46\linewidth]{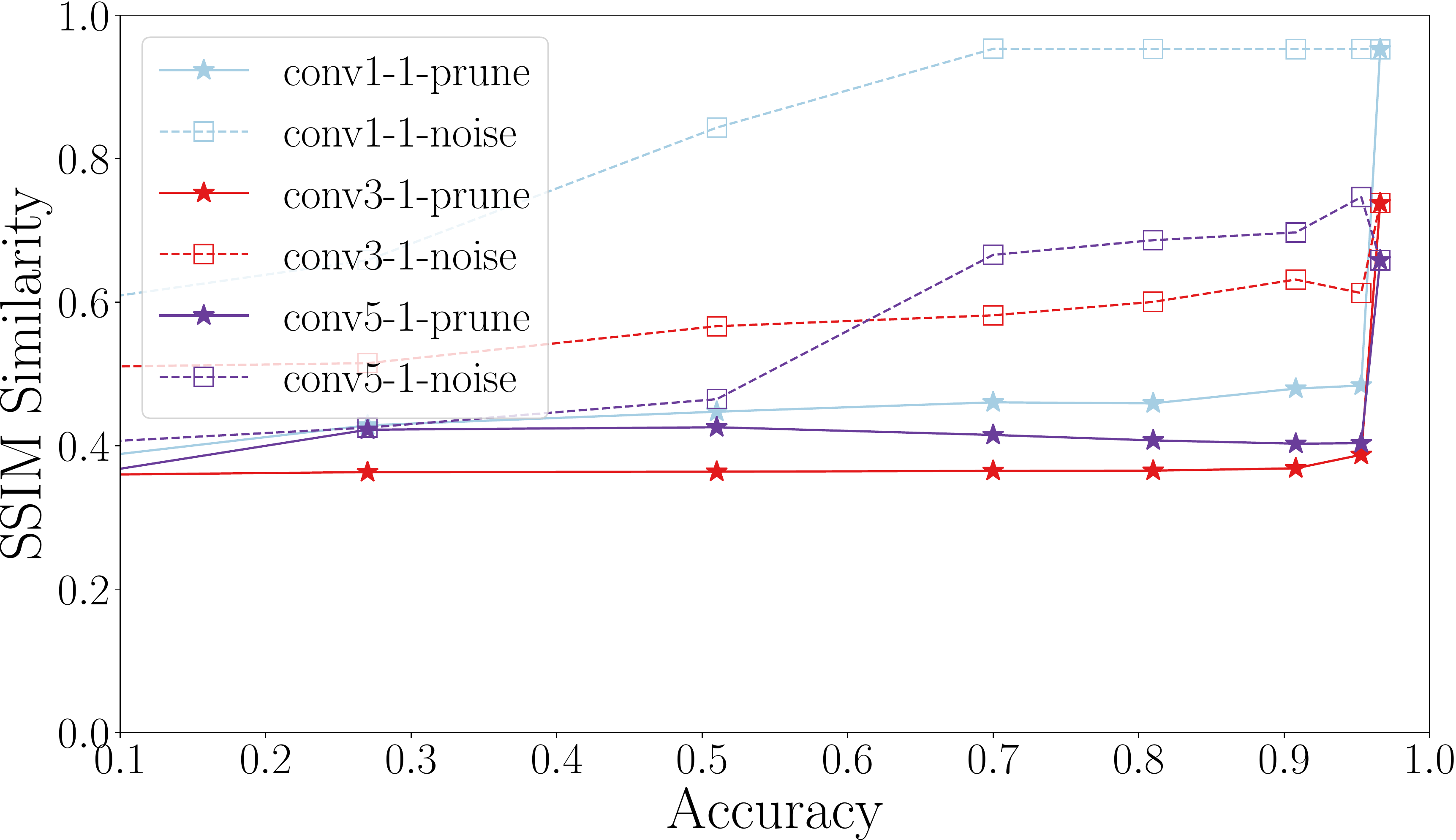}}
        \hspace{4mm}
        \subfloat{\includegraphics[width=0.46\linewidth]{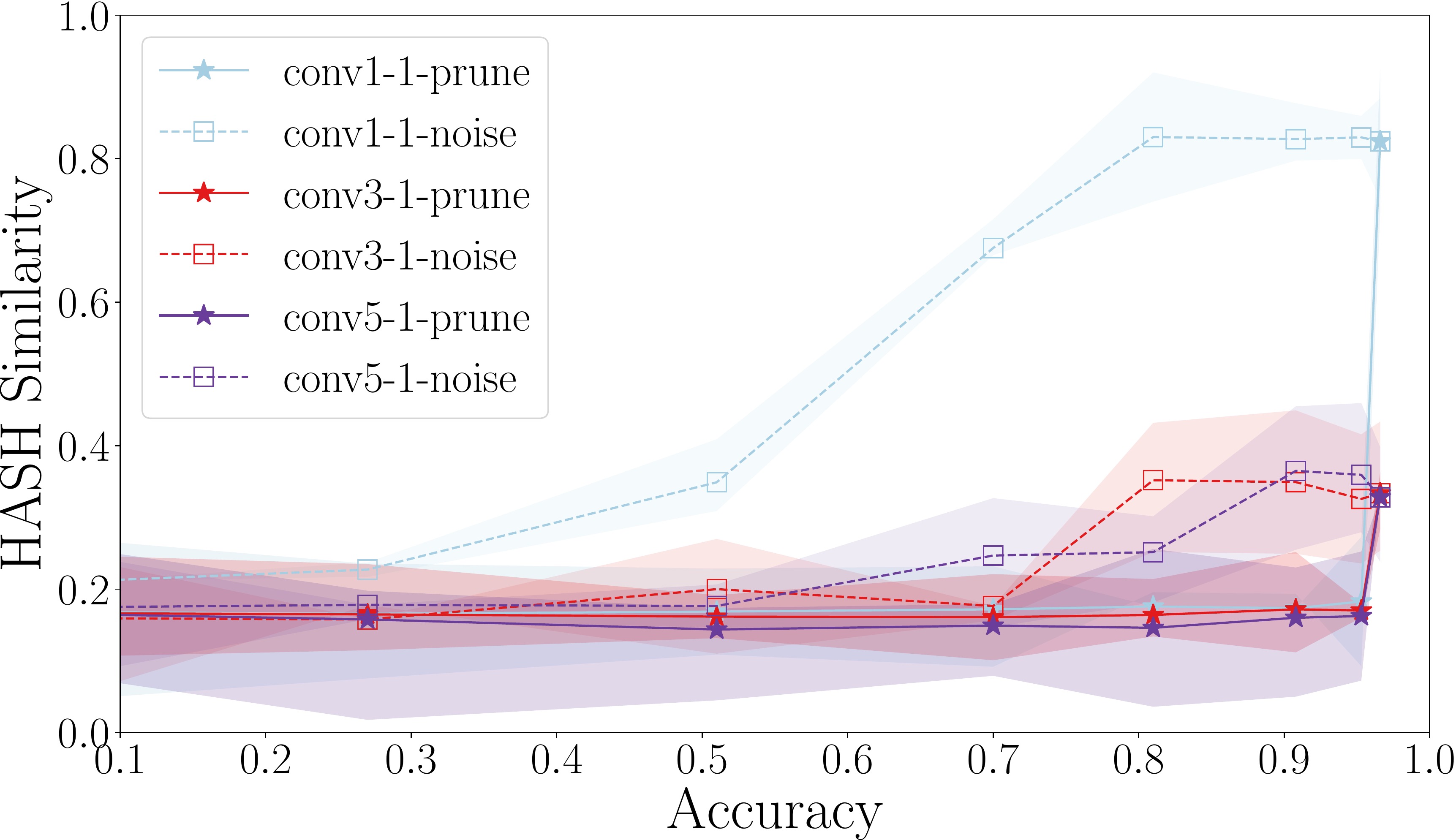}} 
        \hspace{4mm}
        \subfloat{\includegraphics[width=0.46\linewidth]{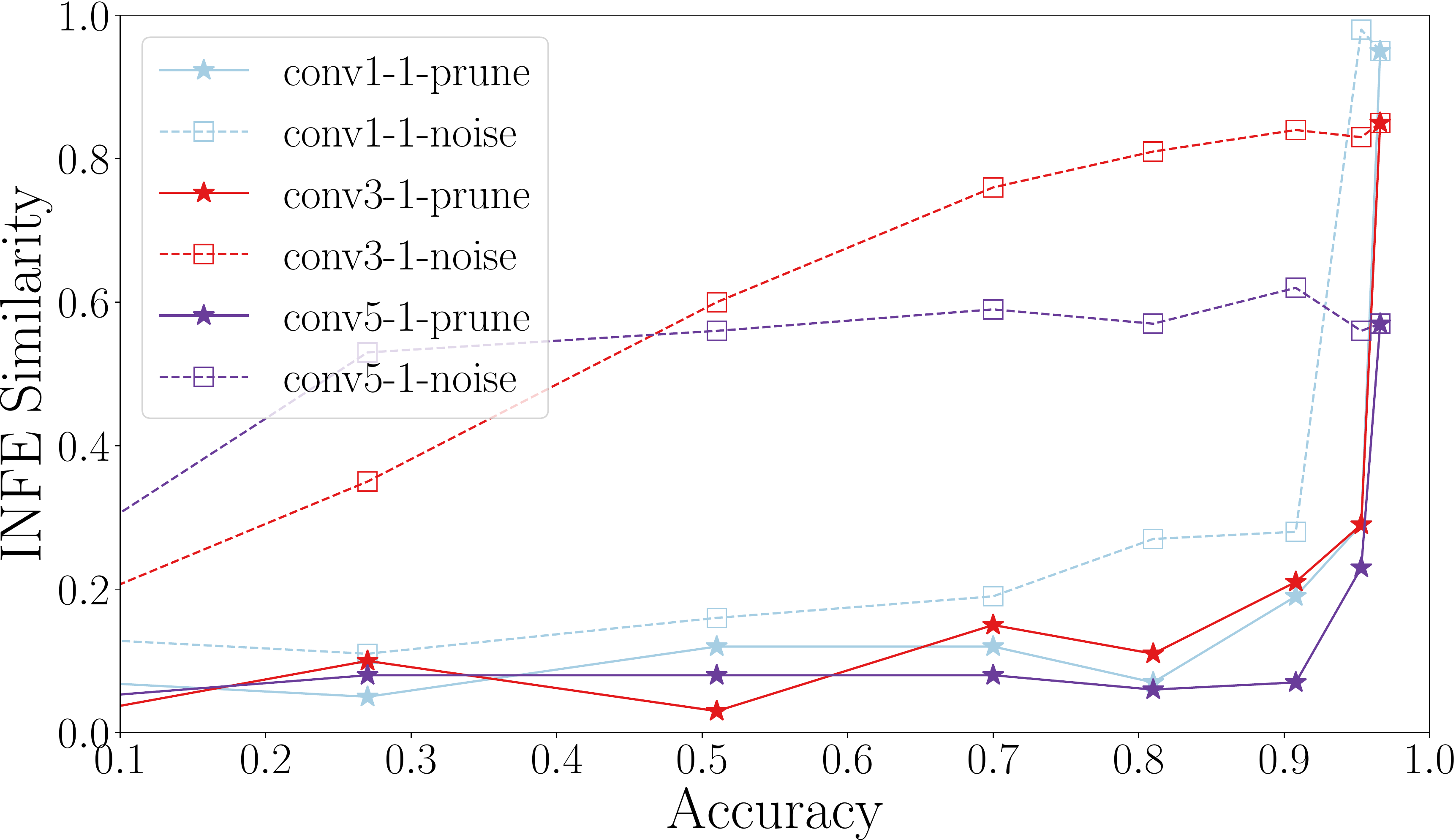}}
    \caption{Similarity between inversion and original images by applying pruning and adding noise on different layers (VGG-19 + CIFAR-10). Shadow represents value range.}
    \label{fig:cifar_4metric}
\end{figure*}

\end{document}